\documentclass[lettersize,journal]{IEEEtran}
\usepackage{amsmath,amsfonts,amssymb}
\usepackage{algorithmic}
\usepackage{algorithm}
\usepackage{array}
\usepackage{textcomp}
\usepackage{stfloats}
\usepackage{url}
\usepackage{verbatim}
\usepackage{graphicx}
\usepackage{cite}

\usepackage{xcolor}
\usepackage{booktabs}
\usepackage{makecell}
\usepackage{tabularx,utfsym,colortbl,arydshln}
\usepackage[caption=false, font=footnotesize]{subfig}
\usepackage{multirow}

\newcommand{\methodname}{FedSAF}

\newcommand{\ie}{\textit{i.e.}}
\newcommand{\eg}{\textit{e.g.}}

\usepackage{amsthm}

\newtheorem{proposition}{Proposition}

\theoremstyle{remark}
\newtheorem{remark}{Remark}

\begin{document}

\title{From Coordinate Matching to Structural Alignment: Rethinking Prototype Alignment in Heterogeneous Federated Learning}

\author{Xinghao~Wu, Jianwei~Niu,~\IEEEmembership{Fellow,~IEEE}, Guogang~Zhu, Xuefeng~Liu, Shaojie~Tang and Jiayuan Zhang
		\IEEEcompsocitemizethanks{\IEEEcompsocthanksitem Xinghao Wu, Guogang Zhu, and Jiayuan Zhang are with the State Key Laboratory of Virtual Reality Technology and Systems, School of Computer Science and Engineering, Beihang University, Beijing 100091, China (e-mail: wuxinghao@buaa.edu.cn,  buaa\_zgg@buaa.edu.cn, zhangjiayuan@buaa.edu.cn).
            \IEEEcompsocthanksitem Jianwei Niu and Xuefeng Liu are with the State Key Laboratory of Virtual Reality Technology and Systems, School of Computer Science and Engineering, Beihang University, Beijing 100091, China, and also with Zhongguancun Laboratory, Beijing 100194, China (e-mail: niujianwei@buaa.edu.cn, liu\_xuefeng@buaa.edu.cn).
			\IEEEcompsocthanksitem Shaojie Tang is with the Center for AI Business Innovation, Department of Management Science and Systems, School of Management, University at Buffalo, NY 14260 USA (e-mail: shaojiet@buffalo.edu). 
			\IEEEcompsocthanksitem Corresponding author: Xuefeng Liu.
	}}

\markboth{Journal of \LaTeX\ Class Files,~Vol.~14, No.~8, August~2021}%
{Shell \MakeLowercase{\textit{et al.}}: A Sample Article Using IEEEtran.cls for IEEE Journals}


\maketitle

\begin{abstract}
Heterogeneous federated learning (HtFL) aims to enable collaboration among clients that differ in both data distributions and model architectures. Prototype-based methods, which communicate class-level feature centers (prototypes) instead of full model parameters, have recently shown strong potential for HtFL. Existing prototype-based HtFL methods typically reuse the MSE-based or cosine-based alignment mechanism developed for homogeneous FL when aligning client-specific representations with global prototypes. These approaches are essentially coordinate alignment, where representations of clients are forced to match the global prototypes in the embedding space in an element-wise manner. Such alignment implicitly assumes that all clients should map their representations into the feature subspace\footnote{In this paper, we use \emph{feature subspace} in an informal sense to denote a (typically low-dimensional) representation space induced by a model; it is \emph{not} required to be a linear subspace in the strict algebraic sense. We use \emph{feature space} and \emph{feature subspace} interchangeably unless otherwise specified.} defined by the global prototypes. This assumption is reasonable in homogeneous FL, where all clients share the same feature extractor. However, it becomes problematic in HtFL, since heterogeneous feature extractors naturally induce client-specific feature subspaces, and forcing all clients to optimize within a single global subspace unnecessarily suppresses their learning capacity.
We observe that coordinate alignment implicitly couples two distinct objectives: aligning inter-class semantic structure, which is directly beneficial for classification, and enforcing a shared feature basis, which is unnecessary and even harmful under model heterogeneity. 
Building on this insight, we design \methodname{}, a \underline{Fed}erated \underline{S}tructural \underline{A}lignment \underline{F}ramework for HtFL, which shifts the alignment objective from absolute coordinates to inter-class relational structure. We demonstrate both theoretically and empirically that structural alignment consistently outperforms coordinate alignment in heterogeneous settings. Moreover, our framework is orthogonal to the choice of global prototype construction. Experiments on multiple benchmarks show that, even with vanilla prototype aggregation, our structural alignment already outperforms state-of-the-art prototype-based HtFL methods by up to 3.52\%. 
\end{abstract}


\section{Introduction}
\label{sec:intro}
Federated Learning (FL) \cite{FedAvg,10571602} enables clients to collaboratively train models without exposing their raw data. A primary challenge in FL is data heterogeneity, where the data distributions across clients are non-independent and identically distributed (non-IID). Numerous studies have sought to address this issue: some focus on improving the robustness of a single global model \cite{FedProx,SCAFFOLD,MOON,10696955,FedFA+}, while others design personalized models \cite{FedCAC,FedDecomp,DiversiFed,FedPFT} tailored to local client distributions. However, these approaches typically assume that all clients adopt the same model architecture, which is often impractical in real-world scenarios. For instance, in cross-silo FL, institutions often require customized architectures and may be unable to disclose model details due to intellectual property (IP) constraints. In cross-device FL, resource heterogeneity leads to diverse feasible model sizes across devices.

To overcome this limitation, \underline{\textbf{H}}e\underline{\textbf{t}}erogeneous \underline{\textbf{F}}ederated \underline{\textbf{L}}earning (HtFL) \cite{HtFLlib} has recently emerged as a promising direction, enabling collaboration across clients with both data and model heterogeneity. The key challenge is how to extract global knowledge and enable mutual assistance among clients when direct parameter aggregation is infeasible. Early solutions rely on public datasets and apply knowledge distillation to transfer knowledge \cite{FedMD,lin2020ensemble,zhang2021parameterized}. However, acquiring suitable public datasets is difficult in practice, and domain mismatch often hinders performance. Other approaches introduce auxiliary models or modules as carriers of global knowledge \cite{FedKD,FedGKT}, but these incur substantial computation and communication costs and may alter local model structures. A particularly appealing line of work is prototype-based HtFL \cite{FedProto,FedTGP,FedTSP}, where clients exchange class prototypes (feature centers). This strategy substantially reduces communication and computation overhead compared to full-model distillation or auxiliary models, making it highly practical.

\begin{figure}[t]
		\centerline{\includegraphics[width=\linewidth]{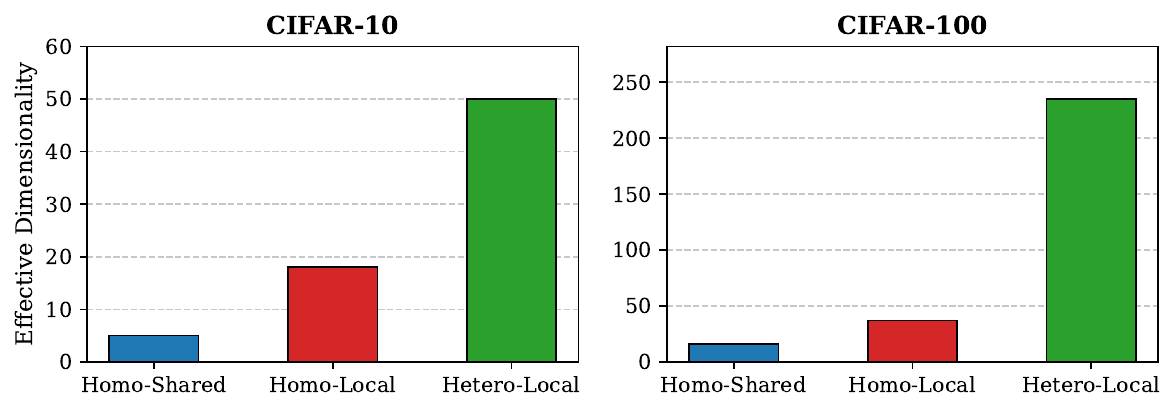}}
        \caption{Comparison of effective dimensionality in homogeneous and heterogeneous FL. We evaluate the feature space discrepancy across three settings: (1) Homo-Shared: clients share both identical model architectures and parameters; (2) Homo-Local: clients share the same architecture but maintain independent parameters; (3) Hetero-Local: clients employ diverse model architectures. The significantly higher effective dimensionality in HtFL suggests that heterogeneous models naturally induce distinct feature subspaces, making rigid coordinate-level alignment less appropriate as it may suppress model individuality.}
		\label{fig:intro_dimension}
\end{figure}

Prototype-based HtFL faces two fundamental challenges: (1) how to construct global prototypes that effectively capture shared characteristics across clients, and (2) how to align local representations with global prototypes so that global knowledge is transferred to local models. Most existing works primarily focus on the first challenge \cite{FedTGP,AlignFed,FedTSP}, while the second has received considerably less attention. Current solutions typically reuse alignment objectives from homogeneous FL, such as L2 distance (e.g., MSE loss)~\cite{FedPAC,FedPLVM} or cosine similarity (e.g., contrastive loss)~\cite{FPL,FedPCL,FedFA}. These approaches are essentially coordinate alignment, where client representations are forced to match the global prototypes in the embedding space in an element-wise manner. This form of alignment implicitly couples two distinct objectives: (i) aligning the inter-class semantic structure encoded by prototypes, and (ii) mapping all client representations into the feature subspace defined by the global prototypes. In homogeneous FL, where all clients share a single global feature extractor and their representations naturally lie in the same feature space, both objectives are jointly satisfied at negligible cost.

In HtFL, however, this assumption no longer holds. As demonstrated in Fig.~\ref{fig:intro_dimension}, we compare the feature space discrepancy across clients under three settings: (1) all clients share the same feature extractor; (2) all clients use the same backbone architecture but maintain independent parameters; and (3) clients employ different feature extractors. We stack prototypes from all clients and perform singular value decomposition (SVD) to compute the effective dimensionality. A higher effective dimensionality indicates larger discrepancies among client feature spaces. Compared with homogeneous FL (Homo-Shared and Homo-Local), the model-heterogeneous scenario (Hetero-Local) exhibits a substantially higher effective dimensionality, suggesting that in HtFL, \emph{clients naturally tend to optimize within their own feature subspaces}. In such a scenario, forcing all clients to optimize within a single global feature subspace can suppress model learning capacity and limit the benefits of collaboration.

\begin{figure}[t]
	\centering
	\subfloat[Coordinate Alignment]{
		\label{coordinate alignment}
		\includegraphics[width=0.48\linewidth]{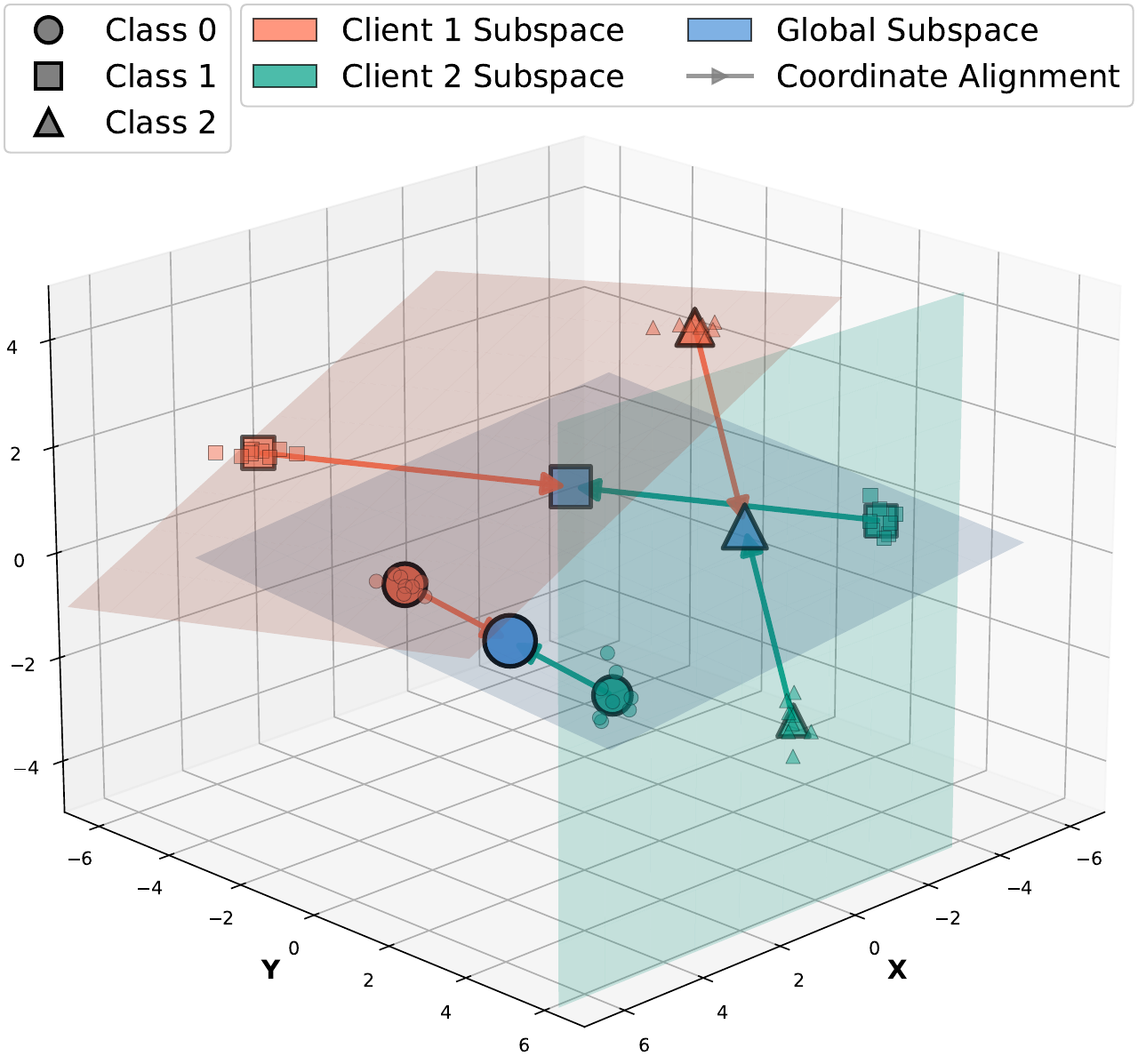}}
	\subfloat[Structural Alignment]{
		\label{structure alignment}
		\includegraphics[width=0.48\linewidth]{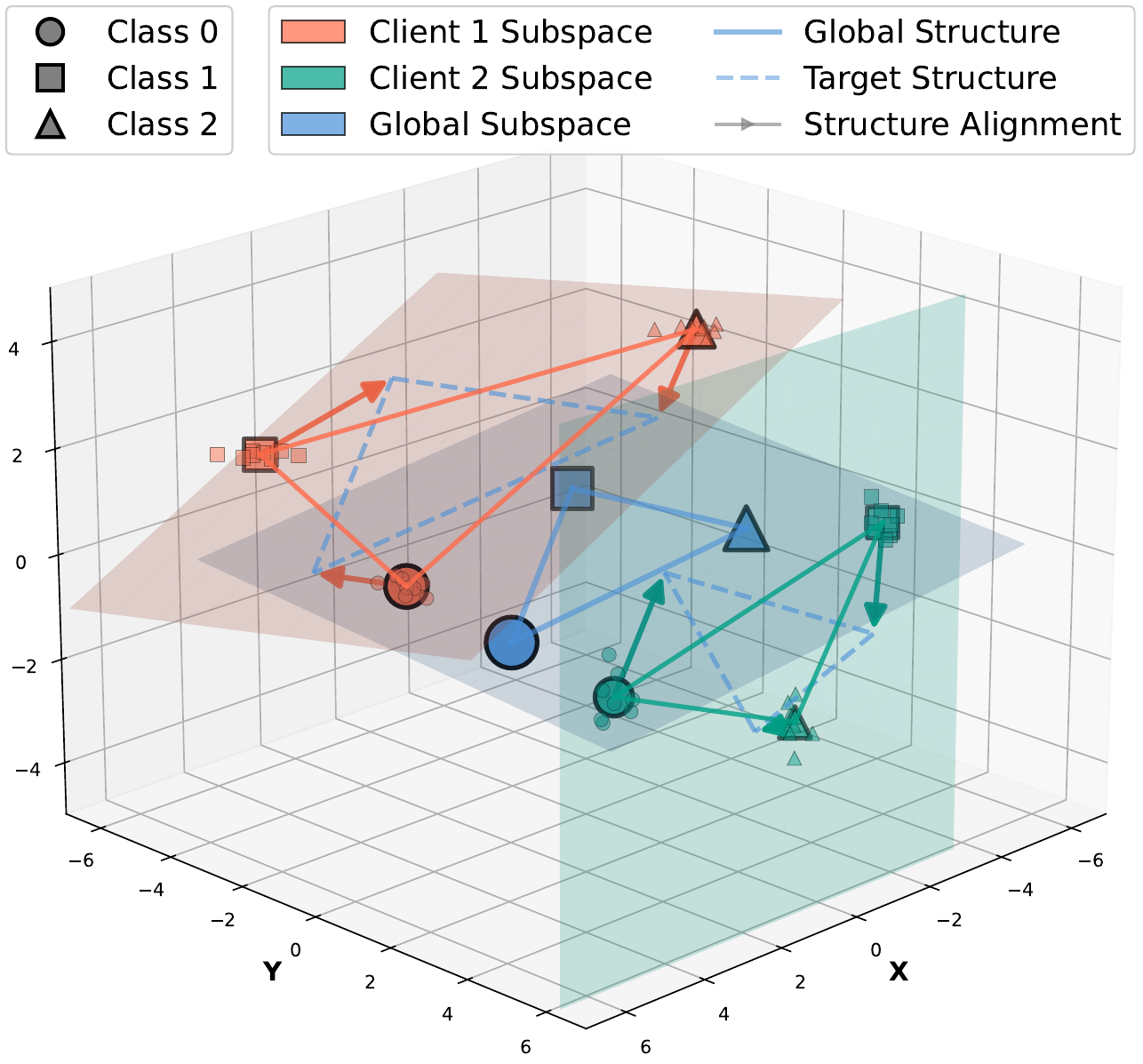}}
	
	\caption{Conceptual comparison between Coordinate Alignment and Structural Alignment in HtFL. (a) \textbf{Coordinate alignment} enforces pointwise matching between client representations and global prototypes in a shared feature space, which may suppress model individuality. (b) Our proposed \textbf{structural alignment} encourages clients to share the semantic structure encoded by inter-class relations while allowing each client to optimize within its own feature subspace}
	\label{fig:intro_alignment_diff}
\end{figure}

The above analysis suggests that only the first objective, aligning inter-class semantic structure, is essential for effective knowledge transfer, whereas enforcing a shared feature basis is unnecessary and potentially harmful in HtFL. Building on this insight, we propose \methodname{}, a \underline{\textbf{Fed}}erated \underline{\textbf{S}}tructural \underline{\textbf{A}}lignment \underline{\textbf{F}}ramework tailored to HtFL. Unlike coordinate alignment, which couples semantic structure transfer with feature-space enforcement, our framework encourages clients to preserve only the inter-class relational structure conveyed by global prototypes, which is directly tied to classification performance, while allowing each client to optimize within its own feature space. A toy example illustrating the difference between the two paradigms is shown in Fig.~\ref{fig:intro_alignment_diff}. Different colors represent prototypes from different clients, and different markers correspond to different classes. In Fig.~\ref{fig:intro_alignment_diff}(a), each client's local prototypes are forced to align point-to-point with the global prototypes in a shared space (blue plane). In Fig.~\ref{fig:intro_alignment_diff}(b), each client optimizes within its own space (red and green planes), while only the inter-class structure defined by the global prototypes (blue triangle) is preserved.

\begin{figure}[t]
		\centerline{\includegraphics[width=\linewidth]{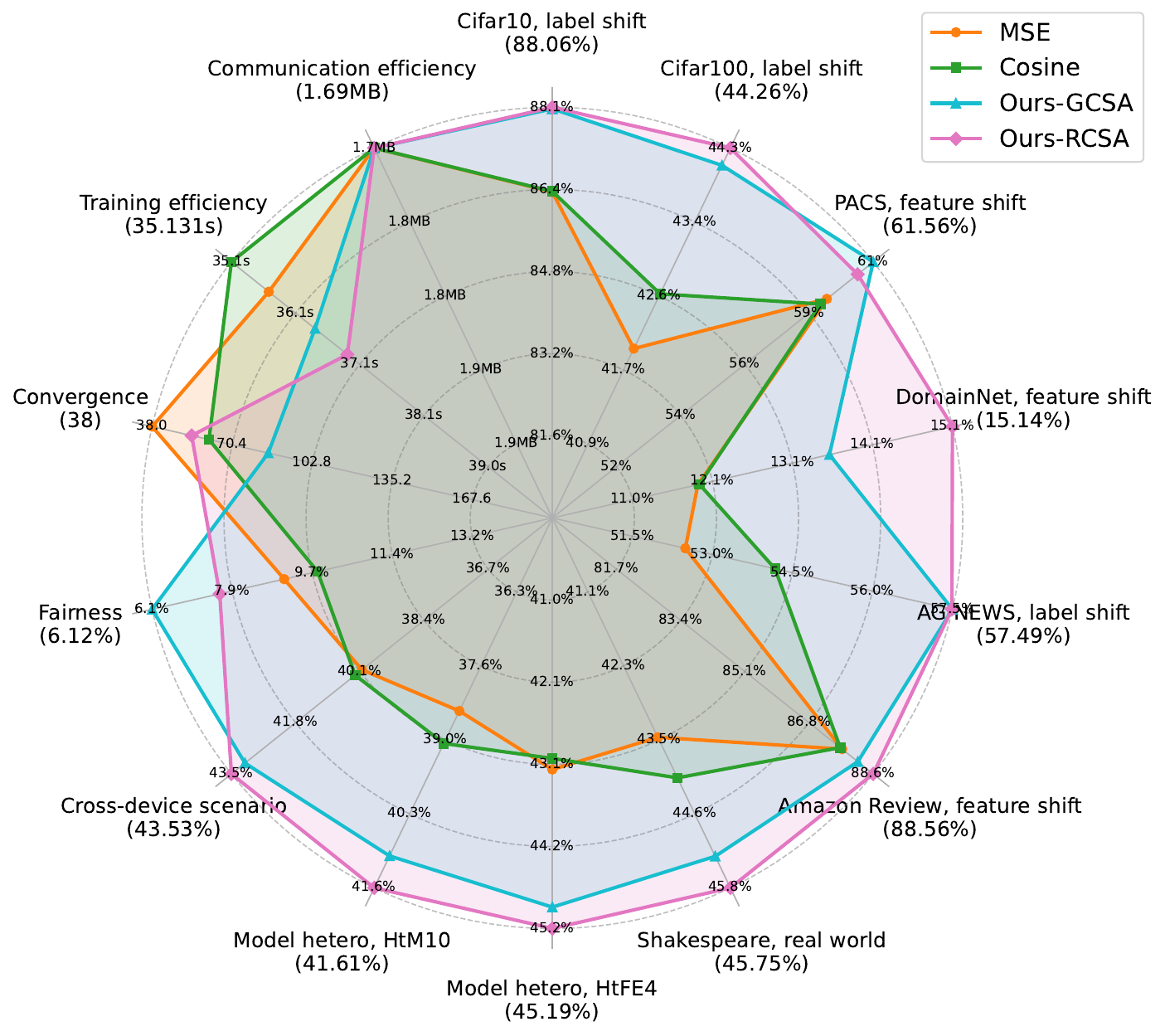}}
        \caption{Comprehensive comparison between coordinate alignment (MSE, Cosine) and structural alignment (Ours-GCSA, Ours-RCSA) across multiple evaluation dimensions, including various datasets, non-IID scenarios, model heterogeneity levels, and efficiency metrics. Our structural alignment consistently outperforms coordinate alignment across nearly all dimensions.}
		\label{fig:comparison_radar}
\end{figure}

To realize this framework, we instantiate it with two effective structural alignment losses. As partially summarized in Fig.~\ref{fig:comparison_radar}, our structural alignment outperforms coordinate alignment across nearly all evaluation dimensions, including diverse datasets, non-IID scenarios (label shift, feature shift, and real-world), multiple modalities (vision and language), varying levels of model heterogeneity, cross-device scalability, and training efficiency. Moreover, our alignment framework is orthogonal to the choice of global prototype construction. Even with vanilla prototype aggregation, our method significantly outperforms state-of-the-art prototype-based HtFL methods, and can be further combined with advanced prototype construction strategies to achieve additional gains.

Our contributions can be summarized as follows:
\begin{itemize}
    \item We identify a critical yet overlooked limitation in existing prototype-based HtFL methods: the reliance on coordinate alignment. We demonstrate, both theoretically and empirically, that forcing heterogeneous clients to align with a single global feature space suppresses the individuality of client-specific feature spaces.
    \item We propose a Structural Alignment Framework that shifts the alignment objective from absolute coordinates to relational geometry. We provide two concrete instantiations, GCSA and RCSA, which effectively transfer global semantic knowledge without compromising the unique feature spaces of local models.
    \item We conduct extensive evaluations across diverse non-IID settings, modalities, and model heterogeneity levels. The results verify that our method significantly outperforms existing baselines. Furthermore, our approach serves as a plug-and-play module that is orthogonal to prototype construction methods, offering a flexible solution for enhancing various HtFL systems.
\end{itemize}
\section{Related Work}
\label{sec:related}
\noindent\textbf{Heterogeneous Federated Learning.} Heterogeneous federated learning (HtFL) enables client collaboration under both data and model heterogeneity, where a key challenge is extracting and transferring global knowledge without direct model aggregation~\cite{HtFLlib,ye2023heterogeneous}. Unlike homogeneous FL that assumes identical model structures, HtFL allows clients to maintain personalized architectures suited to their computational resources and task requirements. 

Existing HtFL methods mainly fall into five categories. \textit{(1) Sub-model extraction methods}\cite{HeteroFL,FedRolex,FIARSE} allow clients to train smaller sub-models extracted from a larger global model, but assume clients train portions of a shared architecture rather than fully independent models. \textit{(2) Knowledge distillation methods}\cite{FedMD,FedGKT,FedGen,FCCL+,RAHFL} transfer knowledge through soft labels or intermediate representations on a shared dataset, often incurring substantial computational overhead or facing challenges when suitable proxy data is unavailable. \textit{(3) Mutual learning methods}\cite{FML,FedMRL,FedType} introduce auxiliary models for bidirectional knowledge transfer, but bring additional computation and communication costs due to auxiliary model training. \textit{(4) Partial parameter sharing methods}\cite{LG-FedAvg,FedGH,FedSSA} assume part of client models remain homogeneous for aggregation, which limits architectural flexibility and may expose partial model structures.

In addition to the above four categories, \textit{prototype-based methods}\cite{FedProto,AlignFed,FedTGP,FedKTL,FedSA,FedTSP} have emerged as an appealing alternative due to their lower computational and communication costs, as clients only need to share class-wise feature centers without requiring public data or auxiliary models. Existing prototype-based HtFL methods primarily differ in how they construct global prototypes. FedProto\cite{FedProto} directly aggregates local prototypes uploaded by clients via weighted averaging. AlignFed~\cite{AlignFed} pre-defines a set of global prototypes as uniformly distributed points on a hypersphere, providing a fixed geometric anchor for alignment. FedTGP~\cite{FedTGP} and FedSA~\cite{FedSA} randomly initialize a set of learnable global prototypes and iteratively update them using local prototypes from clients. FedKTL~\cite{FedKTL} and FedTSP~\cite{FedTSP} take a different approach by extracting global prototypes from pre-trained image generation models or pre-trained language models, leveraging external semantic priors to enrich prototype quality. Despite these advances in prototype construction, the critical question of how to effectively align local representations with global prototypes under heterogeneous feature spaces remains largely underexplored. Most existing methods simply reuse coordinate-level alignment losses (e.g., MSE or cosine similarity) inherited from homogeneous FL, without accounting for the fundamental mismatch between client-specific feature spaces. Our work is orthogonal to these prototype construction strategies and focuses specifically on this overlooked alignment problem.
\section{Methodology}
\subsection{Preliminary}\label{sec:preliminary}
\textbf{Heterogeneous Federated Learning}. We consider an HtFL system with $N$ clients coordinated by a central server. Each client aims to train heterogeneous personalized models $\{w_i \}_{i=1}^N$ to perform well on its own data distribution $\mathbb{D}_i$, where $\mathbb{D}_i \ne \mathbb{D}_j$ for any $i \ne j$. Following FedProto, each $w_i$ is decomposed into a feature extractor $f_i: \mathcal{X} \rightarrow \mathbb{R}^d$, parameterized by $\theta_i$, and a classifier $h_i: \mathbb{R}^d \rightarrow \mathbb{R}^C$, parameterized by $\phi_i$. Here, $\mathcal{X}$ denotes the input space (\eg, raw images). $d$ is the feature dimension, and $C$ is the number of classes. 
The local learning objective on client $i$ is to minimize its empirical loss $\mathcal{L}_D$ over the local distribution $\mathbb{D}_i$. The global HtFL objective can thus be written as:
\begin{equation}
    \min_{\{\theta_i, \phi_i \}_{i=1}^N} \sum_{i=1}^N \frac{1}{N} \mathcal{L}_{D}(\theta_i, \phi_i; \mathbb{D}_i).
\end{equation}

\textbf{Prototype-Based HtFL}. Prototype-based methods maintain a server-side set of class prototypes $\mathcal{P}_g=\{\mathcal{P}_g^c \in \mathbb{R}^d \}_{c=1}^{C}$, where  $\mathcal{P}_g^c$ denotes the global semantic representation of class $c$. During local training, client $i$ uses $\mathcal{P}_g$ as class-level guidance by aligning its representations $z=f_i(x; \theta_i)$ with the corresponding prototype $\mathcal{P}_g^y$. A typical local objective is
\begin{equation}
    \mathcal{L}_i = \mathcal{L}_D(\theta_i, \phi_i; x,y) + \lambda \mathcal{R}(f_i(x; \theta_i), \mathcal{P}_g^y), \text{where} \ x,y \sim \mathcal{D}_i^{\text{train}},
\end{equation}
where $\mathcal{R}$ denotes an alignment loss, $\mathcal{D}_i^{\text{train}}$ denotes the local training set, and $\lambda$ balances supervised learning and prototype guidance.
While existing work largely focuses on improving how $\mathcal{P}_g$ is constructed, our work addresses a complementary and equally important question: how to design the alignment function $\mathcal{R}$ so that it effectively transfers prototype knowledge to clients with heterogeneous feature spaces.


\textbf{Global Prototype Construction}.  
Existing prototype-based HtFL methods mainly differ in how they obtain $\mathcal{P}_g$. A canonical baseline is aggregation (\eg, FedProto): each client forms local class prototypes $\mathcal{P}_i^c$ for class $c$, and the server aggregates them as
\begin{equation}
    \mathcal{P}_g^c
    =
    \frac{\sum_{i=1}^N \omega_i \mathcal{P}_i^c}{\sum_{i=1}^N \omega_i},
\end{equation}
where $\omega_i$ is a client-specific weight (\eg, proportional to its local sample size). Other approaches \cite{AlignFed,FedNH} construct fixed prototypes from external priors,
\begin{equation}
    \mathcal{P}_g^c = a^c,
\end{equation}
or treat $\mathcal{P}_g^c$ as trainable parameters optimized on the server \cite{FedTGP,FedSA}, for example,
\begin{equation}
    \min_{\{\mathcal{P}_g^c\}}
    \sum_{i=1}^N
    \sum_{c \in \mathcal{C}_i}
    \mathcal{L}_{\text{proto}}\big(
        \mathcal{P}_i^c, 
        \mathcal{P}_g^c
    \big),
\end{equation}
where $\mathcal{L}_{\text{proto}}$ encourages consistency between learnable global prototypes and client local prototypes.

In this work, we remain agnostic to the specific mechanism used to construct $\mathcal{P}_g$. We only require that the server can broadcast $\mathcal{P}_g$ each round. In our experiments, we adopt the vanilla aggregation-based prototype construction method (\ie, FedProto) by default.


\textbf{Coordinate Alignment.} 
Most existing prototype-based HtFL methods reuse coordinate-level losses such as MSE or cosine similarity, which enforce element-wise matching between $z$ and $\mathcal{P}_g^y$:
\begin{equation}
    \mathcal{L}_{\text{coord}}(z, \mathcal{P}_g^y)
    =
    \mathcal{D}_{\text{coord}}(z, \mathcal{P}_g^y),
\end{equation}
where $\mathcal{D}_{\text{coord}}$ is instantiated as an $\ell_2$ loss
$ \| z - \mathcal{P}_g^y \|_2^2 $
or a cosine-based loss enforcing directional colinearity. This design is natural when all representations reside in a shared feature subspace.
However, under model heterogeneity, different encoders can realize the same semantics in different (feature) coordinate bases. The next subsection makes this hidden coupling explicit and motivates a shift from coordinate matching to geometry matching.

\subsection{Motivation: Coordinate Alignment Couples Geometry Matching and Basis Matching}
\label{sec:coord-coupling}

Coordinate alignment operates in the \emph{raw feature coordinates}, and thus implicitly assumes that client representations and global prototypes are expressed in a shared feature subspace (\ie, a shared basis). This assumption is often acceptable in homogeneous FL, but becomes fragile in HtFL where heterogeneous feature extractors naturally induce client-specific feature spaces related by unknown transformations.

To expose what coordinate alignment actually enforces, we analyze its batch form.
Given a mini-batch of $n$ samples, let $Z\in\mathbb{R}^{n\times d}$ stack client representations (row-wise) and let $P\in\mathbb{R}^{n\times d}$ stack their corresponding global prototypes.
Let $\hat Z$ and $\hat P$ be the row-wise $\ell_2$-normalized versions of $Z$ and $P$.\footnote{With row-wise normalization, MSE minimization on $\hat Z,\hat P$ is equivalent to cosine-similarity maximization up to constants.}

\begin{proposition}[Implicit Decomposition of Coordinate Alignment]
\label{prop:coord_decomposition}
Let
\begin{equation}
\mathcal{L}_{\mathrm{coord}}(Z,P):=\|\hat Z-\hat P\|_F^2.
\end{equation}
Let $R^\star$ be the (orthogonal) Procrustes solution
\begin{equation}
R^\star \in \arg\min_{R\in\mathcal{O}(d)} \|\hat Z-\hat P R\|_F^2
\ \Longleftrightarrow\
R^\star \in \arg\max_{R\in\mathcal{O}(d)} \langle \hat Z,\hat P R\rangle,
\end{equation}
where $\langle A,B\rangle:=\mathrm{tr}(A^\top B)$ is the Frobenius inner product.
Then $\mathcal{L}_{\mathrm{coord}}$ admits the exact decomposition
\begin{equation}
\label{eq:coord_decomp}
\mathcal{L}_{\mathrm{coord}}(Z,P)
=
\underbrace{\min_{R\in\mathcal{O}(d)} \|\hat Z-\hat P R\|_F^2}_{\mathcal{L}_{\mathrm{shape}}}
+
\underbrace{2\Big(\langle \hat Z,\hat P R^\star\rangle - \langle \hat Z,\hat P\rangle\Big)}_{\mathcal{L}_{\mathrm{rigid}}}.
\end{equation}
Moreover, $\mathcal{L}_{\text{rigid}}\ge 0$, and $\mathcal{L}_{\text{rigid}}=0$ iff the identity rotation is optimal for the Procrustes problem, \ie,
$\langle \hat{Z}, \hat{P} \rangle = \max_{R \in \mathcal{O}(d)} \langle \hat{Z}, \hat{P}R \rangle$.
\end{proposition}

Eq.~\eqref{eq:coord_decomp} reveals that minimizing coordinate alignment implicitly minimizes two coupled objectives:

\textit{(1) Geometry/structure matching ($\mathcal{L}_{\mathrm{shape}}$).}
This term measures the discrepancy between $\hat Z$ and $\hat P$ up to an optimal rotation.
In particular,
\begin{align}
& \mathcal{L}_{\mathrm{shape}}(Z,P)=0 \nonumber \\
&\Longleftrightarrow\
\exists R\in\mathcal{O}(d):\ \hat Z=\hat P R
\ \Longleftrightarrow\
\hat Z\hat Z^\top=\hat P\hat P^\top.
\end{align}
Hence $\mathcal{L}_{\mathrm{shape}}$ captures a rotation-invariant notion of relational geometry
(\eg, Gram-matrix relations), which is the desirable component for transferring class relations across clients.

\textit{(2) Rigid basis-matching penalty ($\mathcal{L}_{\mathrm{rigid}}$).}
This term is the extra cost incurred by enforcing alignment in a fixed global basis (effectively $R=I$),
instead of allowing the optimal rotation $R^\star$.
In homogeneous FL, feature spaces are naturally aligned and typically $R^\star\approx I$,
making $\mathcal{L}_{\mathrm{rigid}}$ negligible, as evidenced by the low effective dimensionality in the Homo-Shared and Homo-Local settings in Fig.~\ref{fig:intro_dimension}.
In HtFL, however, $R^\star$ can deviate significantly from $I$, so driving $\mathcal{L}_{\mathrm{coord}}$ toward zero implicitly pressures the client to ``bend'' its representation toward the global basis, potentially suppressing model individuality.

This analysis points to a more suitable objective for HtFL: transferring knowledge by matching
\emph{relations/geometry} rather than enforcing pointwise coordinate equality.
In what follows, we formalize a structure-alignment framework that isolates the geometry-matching component while avoiding rigid basis matching.

\subsection{Structure-Alignment Framework}\label{sec:structural alignment}
Motivated by the above decomposition, we generalize prototype alignment from pointwise coordinate matching to structure-level matching.

Let $P \in \mathbb{R}^{n \times d}$ and $Q \in \mathbb{R}^{n \times d}$ denote two sets of vectors arranged as matrices, where each row is a prototype or an instance feature. We define a structure operator 
\begin{equation}
    \mathcal{S}: \mathbb{R}^{n \times d} \rightarrow \mathcal{Z},
\end{equation}
which maps a set of representations to a structural descriptor $S(P)$ that captures relations among rows (\eg, pairwise distances/angles, Gram matrices, or covariance statistics). Structural alignment is then defined as
\begin{equation}
    \mathcal{L}_{\text{struct}}(P, Q) 
    = \mathcal{D}\big(\mathcal{S}(P), \mathcal{S}(Q)\big),
    \label{eq:struct_loss_general}
\end{equation}
where $\mathcal{D}$ is a discrepancy measure defined in the structural space $\mathcal{Z}$ (\eg, Frobenius MSE between matrices, or $1$ minus cosine similarity between vectorized structure descriptors).
This formulation decouples alignment from absolute coordinates: $\mathcal{L}_{\text{struct}}$ can remain small even when $P$ and $Q$ differ by client-specific transformations, as long as their internal semantic geometry is consistent.

\textbf{Instantiations of Structural Alignment.}
Different choices of the structure operator $\mathcal{S}$ and the discrepancy measure $\mathcal{D}$ yield different losses. In this work, we consider the following two concrete instantiations.

\emph{(1) Gram-Cosine structural alignment (GCSA).}
Let $\bar p:=\frac{1}{n}\mathbf{1}^\top P\in\mathbb{R}^{1\times d}$ be the mean row vector,
$P_c := P-\mathbf{1}\bar p$ be the centered matrix, and define the centered Gram matrix
$K_P := P_c P_c^\top$.
We measure similarity via cosine similarity between Gram matrices:
\begin{equation}
    \mathcal{L}_{\mathrm{GCSA}}(P, Q)
    =
    1 - 
    \frac{
        \langle K_P, K_Q \rangle
    }{
        \lVert K_P \rVert_F \, \lVert K_Q \rVert_F
    }.
\end{equation}
GCSA is invariant to translation (via centering), rotation, and positive scaling, as formalized in Proposition~\ref{prop:gcsa_invariance}.

\begin{figure*}[t]
		\centerline{\includegraphics[width=0.8\linewidth]{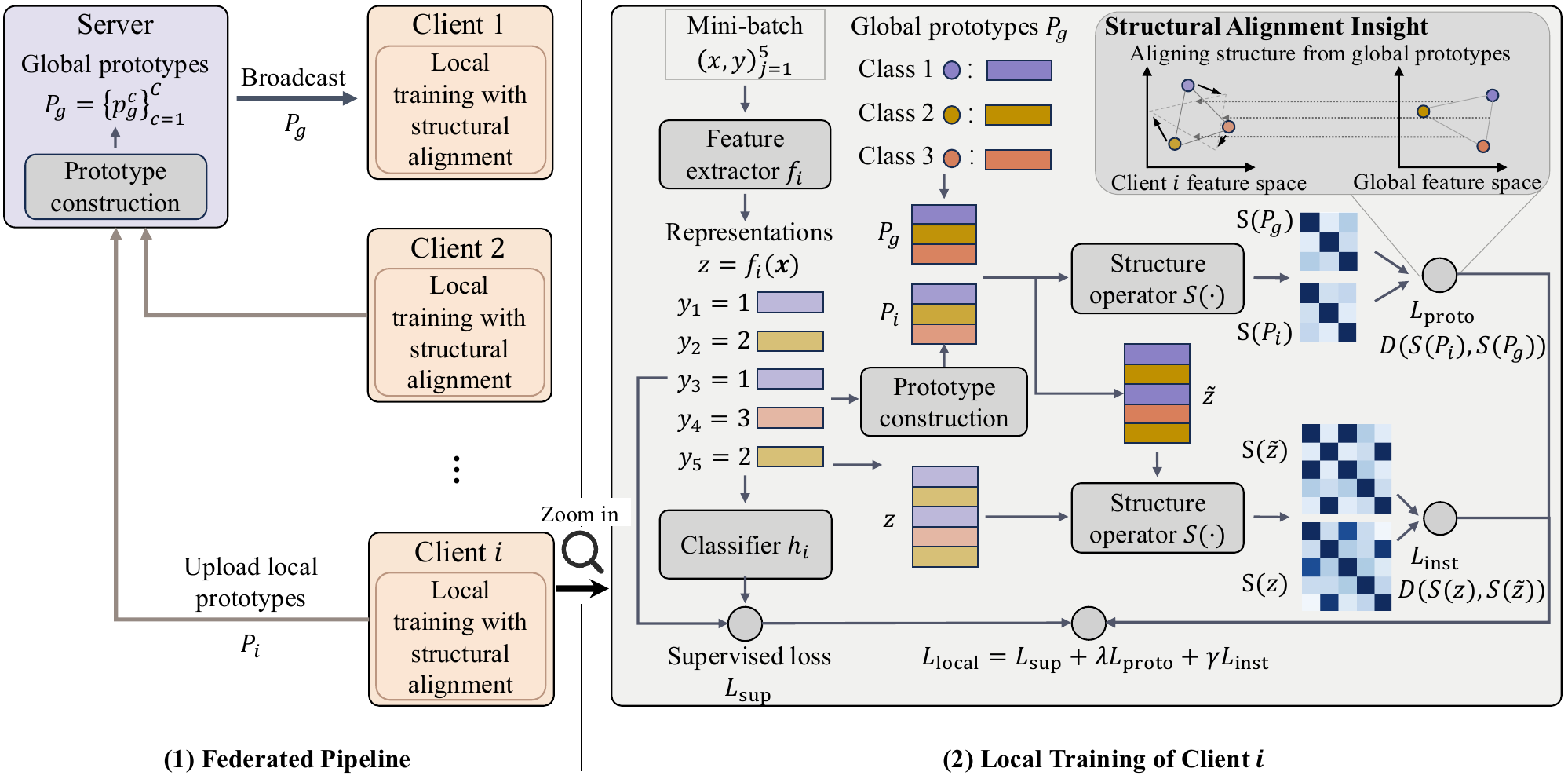}}
        \caption{Overview of \methodname{}. (1) \textbf{Federated Pipeline}: The server constructs global class prototypes $\mathcal{P}_g$ by aggregating uploaded local prototypes and subsequently broadcasts $\mathcal{P}_g$ to heterogeneous clients. (2) \textbf{Local Training on Client $i$}: For a given mini-batch, the client extracts representations $z=f_i(x)$ and forms batch-wise local prototypes $\mathcal{P}_i$. A structure operator $S(\cdot)$ is applied to compute the semantic structures of both local and global prototypes, utilizing a prototype-level loss $\mathcal{L}_{\text{proto}}$ to transfer inter-class relations. The client further assembles $\tilde{z}$ by indexing into $\mathcal{P}_g$ according to the label sequence of $z$, and applies an instance-level loss $\mathcal{L}_{\text{inst}}$ to steer individual sample features toward a geometry compatible with the global prototype structure.}
		\label{fig:overview}
\end{figure*}

\begin{proposition}[Structure invariance of GCSA]
\label{prop:gcsa_invariance}
Let $P, Q \in \mathbb{R}^{n \times d}$.
Suppose $Q$ is obtained from $P$ via an affine feature-space transformation
\begin{equation}
    Q = \alpha P R + \mathbf{1} b^\top, \nonumber
\end{equation}
where $\alpha > 0$, $R^\top R = I$, and $b \in \mathbb{R}^d$.
Then the Gram Cosine structural alignment loss satisfies
\[
    \mathcal{L}_{\mathrm{GCSA}}(P, Q) = 0.
\]
\end{proposition}

\emph{(2) RDM-Cosine structural alignment (RCSA).}
To directly match pairwise relational geometry, we use a representational dissimilarity matrix (RDM). Let $\tilde{P}$ be the row-wise $\ell_2$-normalized version of $P$. We define the squared-distance RDM:
\[
    \mathrm{RDM}_P[i,j] = \big\lVert \tilde{P}[i,:] - \tilde{P}[j,:] \big\rVert_2^2.
\]
We vectorize the upper-triangular entries, denoted $\text{vec}(\text{RDM}_P)$, and apply cosine similarity:
\begin{equation}
    \mathcal{L}_{\mathrm{RCSA}}(P, Q)
    =
    1 -
    \frac{
        \langle \mathrm{vec}(\mathrm{RDM}_P), \mathrm{vec}(\mathrm{RDM}_Q) \rangle
    }{
        \big\lVert \mathrm{vec}(\mathrm{RDM}_P) \big\rVert_2 \, 
        \big\lVert \mathrm{vec}(\mathrm{RDM}_Q) \big\rVert_2
    }.
\end{equation}
RCSA is invariant to orthogonal transformations, as formulated in Proposition~\ref{prop:rcsa_invariance}.

\begin{proposition}[Structure invariance of RCSA]
\label{prop:rcsa_invariance}
Let $P, Q \in \mathbb{R}^{n \times d}$.
Suppose $Q$ is obtained from $P$ via an orthogonal feature-space transformation
\begin{equation}
    Q = P R, \nonumber
\end{equation}
where $R^\top R = I$.
Then the RDM Cosine structural alignment loss satisfies
\[
    \mathcal{L}_{\mathrm{RCSA}}(P, Q) = 0.
\]
\end{proposition}

\begin{remark}[Structural alignment versus coordinate alignment]
\label{rem:structure_vs_coordinate}
Propositions~\ref{prop:gcsa_invariance} and~\ref{prop:rcsa_invariance} show that both GCSA and RCSA impose no penalty under feature-space transformations naturally induced by heterogeneous models.
In contrast, coordinate alignment incurs nonzero loss under such transformations, as it enforces pointwise matching in a shared feature space with a fixed basis.
Notably, RCSA is invariant to orthogonal transformations but penalizes translations, while GCSA further removes translation effects via centering.
\end{remark}

Both GCSA and RCSA conform to the unified formulation $L_{\mathrm{struct}}(P,Q)=D(S(P),S(Q))$
and instantiate the proposed structure-level alignment paradigm with different structural descriptors.
Next, we show how to deploy this framework in prototype-based HtFL, yielding a simple yet effective
two-level structural alignment objective.

\subsection{Structural Alignment for Prototype-Based HtFL}

\textbf{\methodname{} Workflow.}
We follow the standard prototype-based HtFL protocol. In each communication round, the server maintains a global prototype set $P_g=\{P_g^c\}_{c=1}^C$ (constructed by any rule in Sec.~\ref{sec:preliminary}) and broadcasts it to clients. Each client then performs local updates using its private data and uploads its local class prototypes to the server for updating $P_g$. Fig.~\ref{fig:overview} provides an overview of the overall workflow and the proposed two-level structural alignment used in local training.

\textbf{Local Prototype Construction.}
For client $i$, given a mini-batch $\{(x_j,y_j)\}_{j=1}^{B}$ and representations $z_j=f_i(x_j)$, we compute batch-wise local prototypes as the mean feature vector for each class:
\begin{equation}
    \mathcal{P}_i^c
    =
    \frac{1}{|\{j: y_j = c\}|}
    \sum_{j: y_j = c} z_j.
\end{equation}
Stacking prototypes for classes observed in the batch yields $P_i \in \mathbb{R}^{n_i \times d}$, where $n_i$ is the number of present classes.

\textbf{Prototype-level Structural Alignment.}
Let $\mathcal{C}_i$ denote the set of labels appearing in the current batch. We align the structure of the corresponding local and global prototypes:
\begin{equation}
    \mathcal{P}_i^{\mathcal{C}_i}
    = 
    \{ \mathcal{P}_i^c : c \in \mathcal{C}_i \},
    \qquad
    \mathcal{P}_g^{\mathcal{C}_i}
    =
    \{ \mathcal{P}_g^c : c \in \mathcal{C}_i \},
\end{equation}
\begin{equation}
    \mathcal{L}_{\text{proto}}^{(i)}
    =
    \mathcal{D}
    \big(
        \mathcal{S}(\mathcal{P}_i^{\mathcal{C}_i}),
        \mathcal{S}(\mathcal{P}_g^{\mathcal{C}_i})
    \big).
\end{equation}
This term transfers the global inter-class relations without requiring coordinate-wise matching.

\textbf{Instance-level Structural Alignment.}
Prototype alignment constrains class-wise summaries. To further shape the embedding space at the instance level, we align the structure of batch features to a prototype-induced target. Let $z\in \mathbb{R}^{B\times d}$ be the batch feature matrix whose rows are $z_j$, and define a prototype-substituted matrix $\tilde{z}$ by replacing each row with its corresponding global prototype:
\[
\tilde{z}_j=\mathcal{P}_g^{y_j},\quad \tilde{z}=[\tilde{z}_1;\ldots;\tilde{z}_B].
\]
We then define
\begin{equation}
    \mathcal{L}_{\text{inst}}^{(i)}
    =
    \mathcal{D}
    \big(
        \mathcal{S}(z),
        \mathcal{S}(\tilde{z})
    \big).
\end{equation}
Intuitively, $\mathcal{L}_{\text{inst}}^{(i)}$ steers instance features toward a geometry that is compatible with the global prototype structure, while still leaving the client free to choose its own coordinates.

\textbf{Overall local objective.}
The final local objective combines supervised learning and the two structural alignment terms:
\begin{equation}
    \mathcal{L}_{\text{local}}^{(i)}
    =
    \mathcal{L}_{\text{sup}}^{(i)}
    +
    \lambda \mathcal{L}_{\text{proto}}^{(i)}
    +
    \gamma \mathcal{L}_{\text{inst}}^{(i)}.
\end{equation}
Here $\lambda$ and $\gamma$ control the strengths of prototype-level and instance-level alignment, respectively.
\section{Experiments}
\subsection{Experimental Setup}
\definecolor{c1}{RGB}{189,230,205}
\definecolor{c2}{RGB}{228,238,188}
\definecolor{c3}{RGB}{255,248,197}
\definecolor{c4}{RGB}{238,238,238}
\setlength{\intextsep}{0.1cm} 
\begin{table*}[tb]
	\setlength{\abovecaptionskip}{0.cm}
\setlength{\belowcaptionskip}{-0.cm}
\setlength\tabcolsep{4.0pt}
	\begin{center}
    	\renewcommand\arraystretch{0.8}
 \scriptsize
 \caption{Test accuracy (\%) of different methods under various non-IID scenarios on CIFAR-10, CIFAR-100, and Tiny ImageNet with $\text{HtFE}_9$. The top three results are highlighted as \colorbox{c1}{\textbf{first}}, \colorbox{c2}{second}, and \colorbox{c3}{third}, respectively.}
        \label{tab:effect of data hetero}
		\begin{tabular}{@{}cccccccccc@{}}
			\toprule
			&  \multicolumn{3}{c}{CIFAR-10}             &  \multicolumn{3}{c}{CIFAR-100}              & \multicolumn{3}{c}{Tiny ImageNet}           \\ \midrule
			Methods &  $\alpha=0.1$ & $\alpha=0.5$ & $\alpha=1.0$ &  $\alpha=0.1$ & $\alpha=0.5$ & $\alpha=1.0$ & $\alpha=0.1$ & $\alpha=0.5$ & $\alpha=1.0$ \\ \midrule
             \multicolumn{10}{c}{\text{Prototype-based HtFL Methods}} \\
             \midrule
            \makecell{FedProto} &  \makecell{85.81\scriptsize$\pm$0.06} & \makecell{61.53\scriptsize$\pm$0.21} & \makecell{54.34\scriptsize$\pm$0.14} & \makecell{41.07\scriptsize$\pm$0.17} & \makecell{24.48\scriptsize$\pm$0.09} & \makecell{18.89\scriptsize$\pm$0.29} &  \makecell{31.52\scriptsize$\pm$0.20} & \makecell{16.96\scriptsize$\pm$0.10} & \makecell{12.51\scriptsize$\pm$0.05} \\
            \makecell{FedTGP} &  \makecell{85.73\scriptsize$\pm$0.03} & \makecell{61.60\scriptsize$\pm$0.31} & \makecell{53.96\scriptsize$\pm$0.25} & \makecell{41.37\scriptsize$\pm$0.01} & \makecell{24.43\scriptsize$\pm$0.17} & \makecell{18.33\scriptsize$\pm$0.13} &  \makecell{31.16\scriptsize$\pm$0.10} & \makecell{15.70\scriptsize$\pm$0.02} & \makecell{11.80\scriptsize$\pm$0.11} \\
            \makecell{AlignFed} &  \makecell{85.80\scriptsize$\pm$0.34} & \makecell{62.43\scriptsize$\pm$0.10} & \makecell{56.59\scriptsize$\pm$0.30} & \makecell{41.88\scriptsize$\pm$0.05} & \makecell{24.22\scriptsize$\pm$0.23} & \makecell{18.22\scriptsize$\pm$0.26} &  \makecell{30.77\scriptsize$\pm$0.13} & \makecell{14.54\scriptsize$\pm$0.15} & 
            \makecell{10.73\scriptsize$\pm$0.02} \\
            \midrule
            \makecell{\textbf{Ours-GCSA}} &  \makecell{\multicolumn{1}{>{\columncolor{c2}}c}{\text{88.03\scriptsize$\pm$0.13}}} & \makecell{\multicolumn{1}{>{\columncolor{c1}}c}{\textbf{65.95\scriptsize$\pm$0.21}}} & \makecell{\multicolumn{1}{>{\columncolor{c2}}c}{57.29\scriptsize$\pm$0.18}} & \makecell{\multicolumn{1}{>{\columncolor{c2}}c}{\text{44.06\scriptsize$\pm$0.16}}} & \makecell{\multicolumn{1}{>{\columncolor{c2}}c}{26.83\scriptsize$\pm$0.35}} & \makecell{\multicolumn{1}{>{\columncolor{c1}}c}{\textbf{21.04\scriptsize$\pm$0.22}}} &  \makecell{\multicolumn{1}{>{\columncolor{c2}}c}{32.87\scriptsize$\pm$0.19}} & \makecell{\multicolumn{1}{>{\columncolor{c1}}c}{\textbf{18.02\scriptsize$\pm$0.12}}} & \makecell{\multicolumn{1}{>{\columncolor{c1}}c}{\textbf{14.11\scriptsize$\pm$0.08}}} \\
            \makecell{\textbf{Ours-RCSA}} &  \makecell{\multicolumn{1}{>{\columncolor{c1}}c}{\textbf{88.06\scriptsize$\pm$0.24}}} & \makecell{\multicolumn{1}{>{\columncolor{c2}}c}{\text{65.63\scriptsize$\pm$0.19}}} & \makecell{\multicolumn{1}{>{\columncolor{c1}}c}{\textbf{57.57\scriptsize$\pm$0.32}}} & \makecell{\multicolumn{1}{>{\columncolor{c1}}c}{\textbf{44.26\scriptsize$\pm$0.11}}} & \makecell{\multicolumn{1}{>{\columncolor{c1}}c}{\textbf{26.97\scriptsize$\pm$0.24}}} & \makecell{\multicolumn{1}{>{\columncolor{c2}}c}{20.83\scriptsize$\pm$0.29}} &  \makecell{\multicolumn{1}{>{\columncolor{c1}}c}{\textbf{33.08\scriptsize$\pm$0.15}}} & \makecell{\multicolumn{1}{>{\columncolor{c3}}c}{17.86\scriptsize$\pm$0.16}} & \makecell{\multicolumn{1}{>{\columncolor{c3}}c}{13.41\scriptsize$\pm$0.09}} \\
            \midrule
            
            \multicolumn{10}{c}{\text{Other HtFL Methods}} \\
            \midrule
            \makecell{FedGen} &  \makecell{84.42\scriptsize$\pm$0.34} & \makecell{61.63\scriptsize$\pm$0.31} & \makecell{54.54\scriptsize$\pm$0.14} & \makecell{40.64\scriptsize$\pm$0.12} & \makecell{23.20\scriptsize$\pm$0.21} & \makecell{17.86\scriptsize$\pm$0.34} &  \makecell{30.86\scriptsize$\pm$0.22} & \makecell{14.57\scriptsize$\pm$0.06} & 
            \makecell{10.65\scriptsize$\pm$0.15} \\
            \makecell{FedGH} &  \makecell{83.59\scriptsize$\pm$0.16} & \makecell{60.87\scriptsize$\pm$0.26} & \makecell{54.83\scriptsize$\pm$0.06} & \makecell{40.42\scriptsize$\pm$0.17} & \makecell{22.68\scriptsize$\pm$0.21} & \makecell{17.64\scriptsize$\pm$0.14} &  \makecell{29.95\scriptsize$\pm$0.23} & \makecell{12.23\scriptsize$\pm$0.18} & 
            \makecell{\; 9.12\scriptsize$\pm$0.04} \\
            \makecell{LG-FedAvg} &  \makecell{84.53\scriptsize$\pm$0.13} & \makecell{61.31\scriptsize$\pm$0.09} & \makecell{55.26\scriptsize$\pm$0.11} & \makecell{42.56\scriptsize$\pm$0.23} & \makecell{25.03\scriptsize$\pm$0.16} & \makecell{19.25\scriptsize$\pm$0.24} &  \makecell{32.30\scriptsize$\pm$0.19} & \makecell{15.75\scriptsize$\pm$0.06} & 
            \makecell{11.79\scriptsize$\pm$0.05} \\
            \makecell{FML} &  \makecell{86.79\scriptsize$\pm$0.18} & \makecell{63.53\scriptsize$\pm$0.34} & \makecell{57.09\scriptsize$\pm$0.13} & \makecell{42.49\scriptsize$\pm$0.36} & \makecell{25.32\scriptsize$\pm$0.17} & \makecell{19.46\scriptsize$\pm$0.12} &  \makecell{32.43\scriptsize$\pm$0.07} & \makecell{\multicolumn{1}{>{\columncolor{c2}}c}{17.99\scriptsize$\pm$0.05}} & \makecell{\multicolumn{1}{>{\columncolor{c2}}c}{13.85\scriptsize$\pm$0.03}} \\
            \makecell{FedDistill} &  \makecell{85.93\scriptsize$\pm$0.11} & \makecell{62.35\scriptsize$\pm$0.15} & \makecell{55.75\scriptsize$\pm$0.14} & \makecell{42.47\scriptsize$\pm$0.16} & \makecell{\multicolumn{1}{>{\columncolor{c3}}c}{25.50\scriptsize$\pm$0.09}} & \makecell{\multicolumn{1}{>{\columncolor{c3}}c}{19.55\scriptsize$\pm$0.02}} &  \makecell{32.22\scriptsize$\pm$0.32} & \makecell{17.68\scriptsize$\pm$0.09} & \makecell{13.07\scriptsize$\pm$0.08} \\
            \makecell{FedKD} &  \makecell{\multicolumn{1}{>{\columncolor{c3}}c}{86.80\scriptsize$\pm$0.15}} & \makecell{63.86\scriptsize$\pm$0.16} & \makecell{57.05\scriptsize$\pm$0.40} & \makecell{42.54\scriptsize$\pm$0.49} & \makecell{24.78\scriptsize$\pm$0.10} & \makecell{19.20\scriptsize$\pm$0.16} &  \makecell{\multicolumn{1}{>{\columncolor{c3}}c}{32.79\scriptsize$\pm$0.08}} & \makecell{17.80\scriptsize$\pm$0.09} & \makecell{13.72\scriptsize$\pm$0.04} \\
            \makecell{FedMRL} &  \makecell{86.70\scriptsize$\pm$0.19} & \makecell{\multicolumn{1}{>{\columncolor{c3}}c}{64.06\scriptsize$\pm$0.14}} & \makecell{\multicolumn{1}{>{\columncolor{c3}}c}{57.11\scriptsize$\pm$0.30}} & \makecell{\multicolumn{1}{>{\columncolor{c3}}c}{42.82\scriptsize$\pm$0.37}} & \makecell{24.87\scriptsize$\pm$0.23} & \makecell{19.25\scriptsize$\pm$0.29} &  \makecell{31.43\scriptsize$\pm$0.22} & \makecell{16.22\scriptsize$\pm$0.14} & \makecell{12.48\scriptsize$\pm$0.19} \\			
			\bottomrule
		\end{tabular}
	\end{center}
\end{table*}
\textbf{1) Datasets.}
We evaluate \methodname{} on both vision and language benchmarks. 
For image classification, we use CIFAR-10 \cite{Cifar10}, CIFAR-100 \cite{Cifar100}, Tiny-ImageNet \cite{Tiny}, PACS \cite{PACS}, and DomainNet \cite{DomainNet}. 
For text classification, we use AG NEWS \cite{AGNEWSandAmazon}, Amazon Review \cite{AGNEWSandAmazon}, and Shakespeare \cite{Shakespeare}.

\textbf{2) Data Heterogeneity.}
We consider three types of non-IID settings: label-shift, feature-shift, and real-world non-IID.
For label-shift non-IID, we adopt the widely used Dirichlet partitioning scheme, where each client's data is sampled from a Dirichlet distribution $\mathrm{Dir}(\alpha)$. The hyperparameter $\alpha$ controls the degree of data heterogeneity, with a smaller $\alpha$ indicating a higher level of non-IID.
For feature-shift non-IID, we use cross-domain datasets including PACS, DomainNet, and Amazon Review, and assign one domain to each client.
For real-world non-IID, we use the Shakespeare dataset, whose user-wise partitions are naturally non-IID.

\textbf{3) Model Heterogeneity.}
Following HtFLlib \cite{HtFLlib}, we simulate model heterogeneity by instantiating a set of model architectures and assigning them to different clients.
We denote a heterogeneous configuration as ``$\text{HtFE}_X$'' or ``$\text{HtM}_X$'', where ``FE'' means that only the feature extractors are heterogeneous, ``M'' means the entire models are heterogeneous, and ``$X$'' is the number of distinct architectures. 
Each client $i$ is assigned the $(i \bmod X)$-th architecture. 
For image tasks, we consider five heterogeneous settings: $\text{HtFE}_2^{img}$, $\text{HtFE}_4^{img}$, $\text{HtFE}_9^{img}$, $\text{HtM}_4^{img}$, and $\text{HtM}_{10}^{img}$. 
For text tasks, we use $\text{HtFE}_6^{txt}$. 
Unless otherwise specified, $\text{HtFE}_9^{img}$ is used as the default heterogeneous setting.

\textbf{4) Comparison Methods.}
We compare \methodname{} with representative HtFL methods from multiple categories: (1) prototype-based methods including FedProto \cite{FedProto}, FedTGP \cite{FedTGP}, and AlignFed \cite{AlignFed}; (2) knowledge distillation methods including FedGen \cite{FedGen}, FedDistill \cite{FedDistill}; (3) mutual learning methods including FML \cite{FML}, FedMRL \cite{FedMRL}, and FedKD \cite{FedKD}; and (4) partial parameter sharing methods including LG-FedAvg \cite{LG-FedAvg} and FedGH \cite{FedGH}. All implementations are based on HtFLlib \cite{HtFLlib} with identical hyperparameter settings.

\textbf{5) Training Protocol and Metrics.}
By default, we conduct experiments in a cross-silo setting with $N=20$ clients, and all clients participate in every communication round. 
The batch size is set to $100$ and the number of local epochs is $E=5$. 
For state-of-the-art comparisons, we run each method for $300$ global rounds to ensure convergence. 
At each round, we compute the average test accuracy across all clients and report the best accuracy over all rounds. 
For \methodname{}, we adopt the vanilla aggregation-based prototype construction (\ie, FedProto) by default and report two instantiations of the proposed structural alignment framework, denoted \methodname-GCSA and \methodname-RCSA, corresponding to the two structural alignment losses described in Section~\ref{sec:structural alignment}.

\subsection{Comparison with State-of-the-Art Methods}\label{sec:compare with sota}
In this subsection, we compare the proposed \methodname{} with representative state-of-the-art (SOTA) HtFL methods under diverse data- and model-heterogeneity settings to systematically evaluate the effectiveness and performance advantages of our framework.

\textbf{1) Performance under Different Data Heterogeneity Scenarios.}
Table~\ref{tab:effect of data hetero} reports results on CIFAR-10/100 and Tiny-ImageNet under Dirichlet label shift with HtFE$_9$. Across all $\alpha$ values, \methodname-GCSA/RCSA consistently outperform prototype-based baselines (FedProto, FedTGP, AlignFed) while retaining the vanilla global prototype construction. This indicates that, under heterogeneous feature spaces, improving how prototypes are aligned can be more effective than further complicating how prototypes are constructed.

Compared with other HtFL baselines that rely on additional auxiliary models or distillation modules, our method attains competitive, and in several cases superior, accuracy while preserving the communication and computation efficiency of prototype-based schemes. By improving the alignment mechanism rather than the prototype construction pipeline, our framework closes much of the performance gap to heavier HtFL methods while maintaining the lightweight communication pattern that makes prototype-based approaches attractive in practice.

\definecolor{c1}{RGB}{189,230,205}
\definecolor{c2}{RGB}{228,238,188}
\definecolor{c3}{RGB}{255,248,197}
\definecolor{c4}{RGB}{238,238,238}
\begin{table}[tb]
\setlength\tabcolsep{1.5pt}
	\begin{center}
 \scriptsize
 \caption{Test accuracy (\%) of different methods under Dir(0.1) partition on CIFAR-10, CIFAR-100, and Tiny ImageNet across various model heterogeneity settings. The top three results are highlighted as \colorbox{c1}{\textbf{first}}, \colorbox{c2}{second}, and \colorbox{c3}{third}, respectively.}
 \label{tab:effect of model hetero}
		\begin{tabular}{@{}ccccccc@{}}
			\toprule
			&  \multicolumn{2}{c}{CIFAR-10}             &  \multicolumn{2}{c}{CIFAR-100}              & \multicolumn{2}{c}{Tiny ImageNet}           \\ \midrule
			Methods &  $\text{HtFE}_2$ &  $\text{HtFE}_4$ &  $\text{HtFE}_2$ &  $\text{HtFE}_4$ &  $\text{HtFE}_2$ &  $\text{HtFE}_4$ \\ \midrule
            \multicolumn{7}{c}{Prototype-based HtFL Methods} \\
            \midrule
            \makecell{FedProto} &  \makecell{87.36\tiny$\pm$0.15} & \makecell{86.89\tiny$\pm$0.07} &  \makecell{44.99\tiny$\pm$0.10} & \makecell{42.14\tiny$\pm$0.13} &  \makecell{31.70\tiny$\pm$0.12} & \makecell{\multicolumn{1}{>{\columncolor{c3}}c}{32.40\tiny$\pm$0.11}} \\
            \makecell{FedTGP} &  \makecell{\multicolumn{1}{>{\columncolor{c3}}c}{88.06\tiny$\pm$0.13}} & \makecell{\multicolumn{1}{>{\columncolor{c3}}c}{87.49\tiny$\pm$0.25}} &  \makecell{45.81\tiny$\pm$0.13} & \makecell{\multicolumn{1}{>{\columncolor{c3}}c}{42.22\tiny$\pm$0.10}} &  \makecell{31.74\tiny$\pm$0.10} & \makecell{31.98\tiny$\pm$0.14} \\
            \makecell{AlignFed} &  \makecell{87.04\tiny$\pm$0.06} & \makecell{85.51\tiny$\pm$0.09} &  \makecell{43.43\tiny$\pm$0.03} & \makecell{41.53\tiny$\pm$0.13} &  \makecell{29.80\tiny$\pm$0.11} & \makecell{29.42\tiny$\pm$0.18} \\
            \midrule
            \makecell{\textbf{Ours-GCSA}}  &  \makecell{\multicolumn{1}{>{\columncolor{c1}}c}{\textbf{88.90\tiny$\pm$0.13}}} & \makecell{\multicolumn{1}{>{\columncolor{c2}}c}{\text{88.71\tiny$\pm$0.18}}} &  \makecell{\multicolumn{1}{>{\columncolor{c2}}c}{\text{46.64\tiny$\pm$0.15}}} & \makecell{\multicolumn{1}{>{\columncolor{c2}}c}{\text{44.92\tiny$\pm$0.18}}} & \makecell{\multicolumn{1}{>{\columncolor{c2}}c}{32.99\tiny$\pm$0.11}} & \makecell{\multicolumn{1}{>{\columncolor{c2}}c}{\text{32.72\tiny$\pm$0.21}}}  \\
            \makecell{\textbf{Ours-RCSA}}  &  \makecell{\multicolumn{1}{>{\columncolor{c2}}c}{\text{88.88\tiny$\pm$0.17}}} & \makecell{\multicolumn{1}{>{\columncolor{c1}}c}{\textbf{88.95\tiny$\pm$0.16}}} & \makecell{\multicolumn{1}{>{\columncolor{c1}}c}{\textbf{46.67\tiny$\pm$0.14}}} &  \makecell{\multicolumn{1}{>{\columncolor{c1}}c}{\textbf{45.19\tiny$\pm$0.11}}} & \makecell{\multicolumn{1}{>{\columncolor{c1}}c}{\textbf{33.13\tiny$\pm$0.09}}} & \makecell{\multicolumn{1}{>{\columncolor{c1}}c}{\textbf{32.93\tiny$\pm$0.17}}}\\
            \midrule
            \multicolumn{7}{c}{Other HtFL Methods} \\
            \midrule
            \makecell{FedGen} &  \makecell{85.18\tiny$\pm$0.12} & \makecell{84.77\tiny$\pm$0.08} &  \makecell{41.56\tiny$\pm$0.11} & \makecell{39.64\tiny$\pm$0.19} &  \makecell{29.28\tiny$\pm$0.03} & \makecell{29.77\tiny$\pm$0.21} \\
            \makecell{FedGH} &  \makecell{83.46\tiny$\pm$0.24} & \makecell{85.26\tiny$\pm$0.22} &  \makecell{41.81\tiny$\pm$0.14} & \makecell{40.26\tiny$\pm$0.19} &  \makecell{29.90\tiny$\pm$0.05} & \makecell{30.44\tiny$\pm$0.18} \\
            \makecell{LG-FedAvg} &  \makecell{85.64\tiny$\pm$0.05} & \makecell{84.95\tiny$\pm$0.22} &  \makecell{43.15\tiny$\pm$0.06} & \makecell{40.83\tiny$\pm$0.15} &  \makecell{30.09\tiny$\pm$0.05} & \makecell{30.29\tiny$\pm$0.11} \\
            \makecell{FML} &  \makecell{87.12\tiny$\pm$0.05} & \makecell{86.73\tiny$\pm$0.08} &  \makecell{42.25\tiny$\pm$0.09} & \makecell{40.01\tiny$\pm$0.21} &  \makecell{29.77\tiny$\pm$0.15} & \makecell{30.77\tiny$\pm$0.20} \\
            \makecell{FedDistill} &  \makecell{87.32\tiny$\pm$0.03} & \makecell{86.86\tiny$\pm$0.14} &  \makecell{43.46\tiny$\pm$0.09} & \makecell{42.14\tiny$\pm$0.17} &  \makecell{31.05\tiny$\pm$0.07} & \makecell{32.01\tiny$\pm$0.14} \\
            \makecell{FedKD} &  \makecell{87.22\tiny$\pm$0.17} & \makecell{86.38\tiny$\pm$0.19} &  \makecell{\multicolumn{1}{>{\columncolor{c3}}c}{45.99\tiny$\pm$0.09}} & \makecell{34.72\tiny$\pm$0.29} &  \makecell{\multicolumn{1}{>{\columncolor{c3}}c}{32.95\tiny$\pm$0.07}} & \makecell{{27.56\tiny$\pm$0.23}} \\
            \makecell{FedMRL} &  \makecell{87.51\tiny$\pm$0.29} & \makecell{87.00\tiny$\pm$0.09} &  \makecell{43.63\tiny$\pm$0.21} & \makecell{{41.31\tiny$\pm$0.19}} &  \makecell{29.02\tiny$\pm$0.20} & \makecell{28.98\tiny$\pm$0.15} \\        
			
			\bottomrule
		\end{tabular}
	\end{center}
\end{table}
\textbf{2) Performance under Different Model Heterogeneity Scenarios.}
Table~\ref{tab:effect of model hetero} fixes Dir$(0.1)$ and varies heterogeneity from $\text{HtFE}_2$ to $\text{HtFE}_4$. Prototype-based baselines become weaker under more severe heterogeneity, but \methodname{} remains the top performer within this family and is consistently competitive against heavier distillation and mutual-learning methods. Notably, as heterogeneity increases, the degradation of \methodname{} is smaller than coordinate-alignment baselines, corroborating our motivation that rigid coordinate matching becomes increasingly restrictive as client feature spaces diverge.

\subsection{Comparison with Coordinate Alignment}\label{sec:compare_alignment_expe}
To isolate the effect of the alignment objective, we conduct controlled comparisons where all components (client models, training protocol, and global prototype construction) are held identical, and only the local alignment loss is varied. We compare two coordinate-alignment baselines (MSE and cosine matching) with two structural instantiations (GCSA and RCSA).

\begin{table}[tb]
\setlength\tabcolsep{1.5pt}
	\begin{center}
 \scriptsize
 \caption{Test accuracy (\%) of different feature alignment loss under label shift non-IID partition on CIFAR-10 and CIFAR-100.}
 \label{tab:alignment diff in label shift}
		\begin{tabular}{@{}ccccccc@{}}
			\toprule
			&  \multicolumn{3}{c}{CIFAR-10}             &  \multicolumn{3}{c}{CIFAR-100}  \\ 
            \midrule
			Methods &  $\alpha=0.1$ & $\alpha=0.5$ & $\alpha=1.0$ &  $\alpha=0.1$ & $\alpha=0.5$ & $\alpha=1.0$ \\ \midrule
            \makecell{MSE} &  \makecell{86.40\tiny$\pm$0.11} & \makecell{62.51\tiny$\pm$0.16} &  \makecell{55.10\tiny$\pm$0.17} & \makecell{41.95\tiny$\pm$0.16} &  \makecell{25.26\tiny$\pm$0.21} & \makecell{19.54\tiny$\pm$0.25} \\
            \makecell{Cosine} &  \makecell{86.41\tiny$\pm$0.15} & \makecell{62.69\tiny$\pm$0.11} &  \makecell{55.73\tiny$\pm$0.13} & \makecell{42.58\tiny$\pm$0.21} &  \makecell{25.55\tiny$\pm$0.27} & \makecell{19.97\tiny$\pm$0.19} \\
            \makecell{GCSA} &   \makecell{88.03\tiny$\pm$0.13} & \makecell{65.95\tiny$\pm$0.21} &  \makecell{57.29\tiny$\pm$0.18} & \makecell{44.06\tiny$\pm$0.16} &  \makecell{26.83\tiny$\pm$0.35} & \makecell{21.04\tiny$\pm$0.22} \\
            \makecell{RCSA} &   \makecell{88.06\tiny$\pm$0.24} & \makecell{65.63\tiny$\pm$0.19} &  \makecell{57.57\tiny$\pm$0.32} & \makecell{44.26\tiny$\pm$0.11} &  \makecell{26.97\tiny$\pm$0.24} & \makecell{20.83\tiny$\pm$0.29} \\
			\bottomrule
		\end{tabular}
	\end{center}
\end{table}

\textbf{1) Performance under Label Shift non-IID Scenarios.}
We first evaluate the four alignment objectives under label-shift non-IID settings with varying levels of data heterogeneity. The results on CIFAR-10 and CIFAR-100 with HtFE$_9$ are reported in Table~\ref{tab:alignment diff in label shift}.

Both structural alignment variants (GCSA and RCSA) consistently and substantially outperform coordinate alignment baselines (MSE and Cosine) across all datasets and heterogeneity levels. On CIFAR-10, our methods improve over the best coordinate baseline by up to 3.26\% ($\alpha=0.5$), with an average gain of 2.25\% across all settings. On CIFAR-100, the improvement reaches 1.48\% ($\alpha=0.1$), with an average gain of 1.32\%. 
We also observe that the advantage of structural alignment is more pronounced under moderate data heterogeneity ($\alpha=0.5$) than under extreme heterogeneity ($\alpha=0.1$). A plausible explanation is that, under severe label imbalance, the limited class coverage per client restricts the expressiveness of inter-class structural relations, partially diminishing the benefit of structure-level alignment. Nevertheless, even in this challenging regime, structural alignment maintains a clear margin over coordinate-based methods.

\begin{table*}[htb]
\centering
\caption{Experimental results on PACS and DomainNet datasets. The best and second-best results are highlighted in \textbf{bold} and \underline{underline}, respectively.}
\label{tab:alignment feature shift}
\small
\begin{tabular}{@{}lcccccccccccc@{}}
\toprule
 & \multicolumn{5}{c}{PACS} & \multicolumn{7}{c}{DomainNet} \\ 
\cmidrule(lr){2-6} \cmidrule(lr){7-13}
Method & Art & Cart. & Photo & Sketch & \textbf{Avg.} & Clip. & Info. & Paint. & Quick. & Real & Sketch & \textbf{Avg.} \\ \midrule
MSE    & 64.11 & \underline{49.80} & 61.95 & 63.68 & 59.89 & 8.21 & 2.08 & 3.23 & 41.34 & 13.86 & 2.50 & 11.87 \\
Cosine & 61.96 & 46.48 & \textbf{65.70} & 64.50 & 59.66 & \textbf{9.49} & 2.09 & 3.20 & 39.60 & 14.18 & 2.75 & \underline{11.89} \\
GCSA   & \underline{64.59} & \textbf{51.37} & 64.16 & \textbf{66.12} & \textbf{61.56} & \underline{9.28} & \underline{2.64} & \textbf{5.98} & \underline{41.83} & \underline{16.80} & \underline{4.84} & 13.56 \\
RCSA   & \textbf{64.92} & 49.27 & \underline{64.31} & \underline{65.51} & \underline{61.00} & 8.28 & \textbf{3.64} & \underline{5.11} & \textbf{45.56} & \textbf{22.29} & \textbf{5.97} & \textbf{15.14} \\ \bottomrule
\end{tabular}
\end{table*}
\textbf{2) Performance under Feature Shift non-IID Scenarios.}
Beyond label shift, we further evaluate structural alignment under feature shift non-IID settings, where data heterogeneity arises from domain discrepancy rather than class imbalance. We conduct experiments on two cross-domain benchmarks: PACS (4 domains) and DomainNet (6 domains), assigning one domain to each client. The results are reported in Table~\ref{tab:alignment feature shift}.

On PACS, both GCSA and RCSA outperform coordinate alignment baselines in terms of average accuracy, with GCSA achieving 61.56\% and RCSA achieving 61.00\%, compared to 59.89\% for MSE and 59.66\% for Cosine. The improvements are consistent across most individual domains, with GCSA obtaining the best performance on Cartoon and Sketch, and RCSA achieving the highest accuracy on Art.
On DomainNet, the advantage of structural alignment becomes even more pronounced. RCSA achieves the best average accuracy of 15.14\%, substantially outperforming MSE (11.87\%) and Cosine (11.89\%) by over 3\%. The gains are particularly notable on challenging domains such as Real, where RCSA improves over Cosine by 8.11\% (from 14.18\% to 22.29\%), and Sketch, where the improvement reaches 3.22\% (from 2.75\% to 5.97\%). GCSA also demonstrates strong performance, achieving the best results on Painting and competitive accuracy on other domains.

These results indicate that structural alignment is especially beneficial when clients exhibit significant feature distribution shifts. In such scenarios, heterogeneous models trained on visually distinct domains naturally develop diverse feature representations. Coordinate alignment, which forces all clients into a shared global feature space, struggles to accommodate these domain-induced variations. In contrast, structural alignment focuses on preserving inter-class relational geometry, allowing each client to maintain domain-specific feature characteristics while still benefiting from global prototype knowledge.

\begin{table}[tb]
	\begin{center}
\footnotesize
 \caption{Test accuracy (\%) of different feature alignment loss under different non-IID settings in textual modality.}
 \label{tab:alignment diff in textual modality}
		\begin{tabular}{@{}cccc@{}}
			\toprule
			Datasets &  AG News             &  Amazon Review & Shakespeare  \\ 
            \midrule
			Scenarios &  Label Shift & Feature Shift & Real-World \\ \midrule
            \makecell{MSE} &  \makecell{52.49\tiny$\pm$0.09} & \makecell{87.73\tiny$\pm$0.03} &  \makecell{43.41\tiny$\pm$0.29}  \\
            \makecell{Cosine} &  \makecell{54.17\tiny$\pm$0.18} & \makecell{87.68\tiny$\pm$0.05} &  \makecell{44.04\tiny$\pm$0.16} \\
            \makecell{GCSA} &  \makecell{57.49\tiny$\pm$0.16} & \makecell{88.15\tiny$\pm$0.03} &  \makecell{45.26\tiny$\pm$0.16} \\
            \makecell{RCSA} &    \makecell{57.48\tiny$\pm$0.21} & \makecell{88.56\tiny$\pm$0.08} &  \makecell{45.75\tiny$\pm$0.21} \\
			\bottomrule
		\end{tabular}
	\end{center}
\end{table}

\textbf{3) Performance in Textual Modality.}
To verify the generality of structural alignment beyond visual tasks, we conduct experiments on three text classification benchmarks covering diverse non-IID scenarios: AG News (label shift), Amazon Review (feature shift across product categories), and Shakespeare (real-world non-IID with natural user partitions). The results are reported in Table~\ref{tab:alignment diff in textual modality}.
Structural alignment consistently outperforms coordinate alignment across all three datasets and non-IID types. On AG News, GCSA and RCSA achieve 57.49\% and 57.48\% respectively, improving over the best coordinate baseline (Cosine, 54.17\%) by 3.31\%. On Amazon Review, RCSA obtains 88.56\%, surpassing MSE (87.73\%) and Cosine (87.68\%) by approximately 0.8\%. On the Shakespeare dataset, which exhibits naturally heterogeneous data distributions arising from different writing styles of characters, RCSA achieves the best accuracy of 45.75\%, outperforming MSE by 2.34\% and Cosine by 1.71\%.
These results demonstrate that the advantage of structural alignment extends beyond image classification to textual modality as well.

\textbf{4) Performance in Cross-device Scenarios.}
Previous experiments focus on cross-silo settings with a moderate number of clients. We now evaluate structural alignment in cross-device scenarios, where the number of clients is significantly larger and each client holds fewer local samples. To simulate the cross-device setting, we set the client number to 50, 100, and 200, with 20\% of clients participating in each round. Considering the limited computational capacity of mobile devices, we set the batch size to 10 and the number of local epochs to 1. The results on CIFAR-100 under Dir(0.1) with HtFE$_9$ are reported in Table~\ref{tab:cross device}.

Structural alignment consistently outperforms coordinate alignment across all client scales. RCSA achieves the best performance in all settings, improving over Cosine by 3.29\% with 50 clients, 2.24\% with 100 clients, and 1.17\% with 200 clients. As the number of clients increases and local data becomes scarcer, all methods experience performance degradation. However, structural alignment maintains a clear advantage throughout, demonstrating robustness under increased client counts and partial participation. These results suggest that structural alignment is well-suited for large-scale cross-device deployments.
\begin{table}[]
	\begin{center}
    \caption{Test accuracy (\%) of different methods in cross-device scenarios on CIFAR-100.}
 \label{tab:cross device}
		\begin{tabular}{cccc}
			\toprule
			Client \# & 50 & 100 & 200 \\
            \midrule
            \makecell{MSE} & \makecell{40.06\scriptsize$\pm$0.18} & \makecell{38.43\scriptsize$\pm$0.08} & \makecell{34.50\scriptsize$\pm$0.10}  \\
            \makecell{Cosine} & \makecell{40.24\scriptsize$\pm$0.18} & \makecell{38.53\scriptsize$\pm$0.11} & \makecell{34.89\scriptsize$\pm$0.17}  \\
            \makecell{GCSA} & \makecell{43.17\scriptsize$\pm$0.16} & \makecell{40.47\scriptsize$\pm$0.10} & \makecell{35.35\scriptsize$\pm$0.17}  \\
            \makecell{RCSA} & \makecell{43.53\scriptsize$\pm$0.13} & \makecell{40.77\scriptsize$\pm$0.06} & \makecell{36.06\scriptsize$\pm$0.10}  \\
			
			\bottomrule
		\end{tabular}
	\end{center}
\end{table}

\begin{table*}[htb]
\setlength\tabcolsep{3pt}
	\begin{center}
 \caption{Test accuracy (\%) of different feature alignment loss under label shift non-IID partition on CIFAR-10 and CIFAR-100 with different model structure.}
 \label{tab:alignment diff in model hetero}
		\begin{tabular}{@{}ccccccccc@{}}
			\toprule
			&  \multicolumn{4}{c}{CIFAR-10}             &  \multicolumn{4}{c}{CIFAR-100}  \\ 
            \midrule
			Methods &  HtFE$_2$ & HtFE$_4$ & HtM$_4$ &  HtM$_{10}$ &  HtFE$_2$ & HtFE$_4$ & HtM$_4$ &  HtM$_{10}$ \\ \midrule
            \makecell{MSE} &  \makecell{87.87\scriptsize$\pm$0.13} & \makecell{87.47\scriptsize$\pm$0.15} &  \makecell{86.59\scriptsize$\pm$0.08} & \makecell{85.55\scriptsize$\pm$0.06} &  \makecell{46.44\scriptsize$\pm$0.12} & \makecell{43.18\scriptsize$\pm$0.21} &  \makecell{40.64\scriptsize$\pm$0.17} & \makecell{38.45\scriptsize$\pm$0.14} \\
            \makecell{Cosine} &  \makecell{87.67\scriptsize$\pm$0.16} & \makecell{87.42\scriptsize$\pm$0.09} &  \makecell{86.41\scriptsize$\pm$0.13} & \makecell{85.91\scriptsize$\pm$0.05} &  \makecell{44.10\scriptsize$\pm$0.10} & \makecell{43.04\scriptsize$\pm$0.07} &  \makecell{39.99\scriptsize$\pm$0.11} & \makecell{39.03\scriptsize$\pm$0.15} \\
            \makecell{GCSA} &  \makecell{88.90\scriptsize$\pm$0.13} & \makecell{88.71\scriptsize$\pm$0.18} &  \makecell{87.89\scriptsize$\pm$0.11} & \makecell{87.69\scriptsize$\pm$0.10} &  \makecell{46.63\scriptsize$\pm$0.15} & \makecell{44.92\scriptsize$\pm$0.18} &  \makecell{42.13\scriptsize$\pm$0.08} & \makecell{41.04\scriptsize$\pm$0.09} \\
            \makecell{RCSA} &  \makecell{88.88\scriptsize$\pm$0.17} & \makecell{88.95\scriptsize$\pm$0.16} &  \makecell{87.93\scriptsize$\pm$0.11} & \makecell{87.55\scriptsize$\pm$0.08} &  \makecell{46.67\scriptsize$\pm$0.14} & \makecell{45.19\scriptsize$\pm$0.11} &  \makecell{42.29\scriptsize$\pm$0.10} & \makecell{41.61\scriptsize$\pm$0.06} \\
			\bottomrule
		\end{tabular}
	\end{center}
\end{table*}

\textbf{5) Performance under Different Model Heterogeneity Scenarios.}
We further compare the two alignment paradigms under varying degrees of model heterogeneity. We consider both feature extractor heterogeneity (\ie, HtFE$_2$ and HtFE$_4$) and entire model heterogeneity (\ie, HtM$_4$ and HtM$_{10}$). The results on CIFAR-10 and CIFAR-100 under Dir(0.1) are reported in Table~\ref{tab:alignment diff in model hetero}.

Structural alignment (GCSA and RCSA) consistently outperforms coordinate alignment (MSE and Cosine) across all heterogeneity levels on both datasets. More importantly, the advantage of structural alignment becomes increasingly pronounced as model heterogeneity grows. From HtFE$_2$ to HtM$_{10}$, the performance gap between structural and coordinate alignment widens consistently. This trend directly supports our central hypothesis: as architectural diversity increases, client feature spaces become more distinct, and enforcing coordinate matching becomes increasingly restrictive. Structural alignment, by focusing on inter-class relational geometry rather than absolute coordinates, gracefully accommodates this diversity and scales to highly heterogeneous scenarios.

\begin{figure}
        \centering  
	\subfloat[MSE]{
		\includegraphics[width=0.49\linewidth]{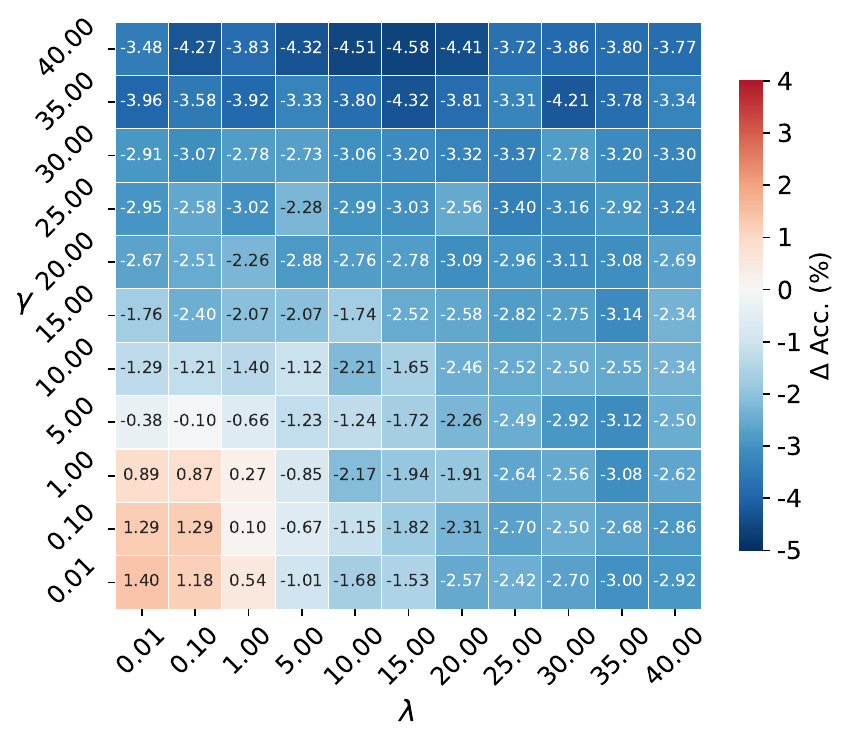}}
    \subfloat[Cosine]{
		\includegraphics[width=0.49\linewidth]{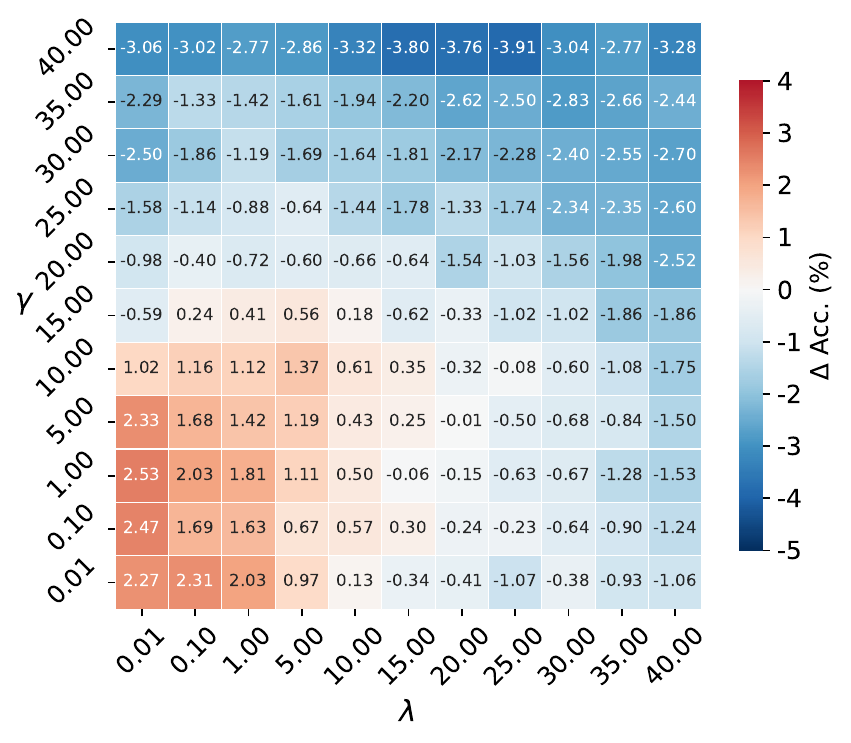}} \\
        \subfloat[GCSA]{
		\includegraphics[width=0.49\linewidth]{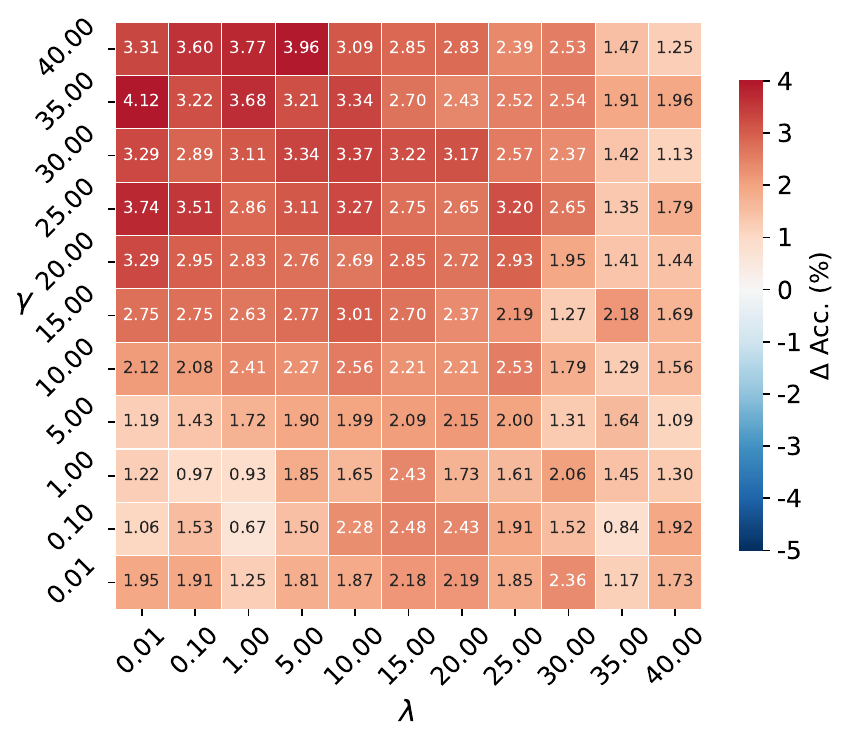}}
    \subfloat[RCSA]{
		\includegraphics[width=0.49\linewidth]{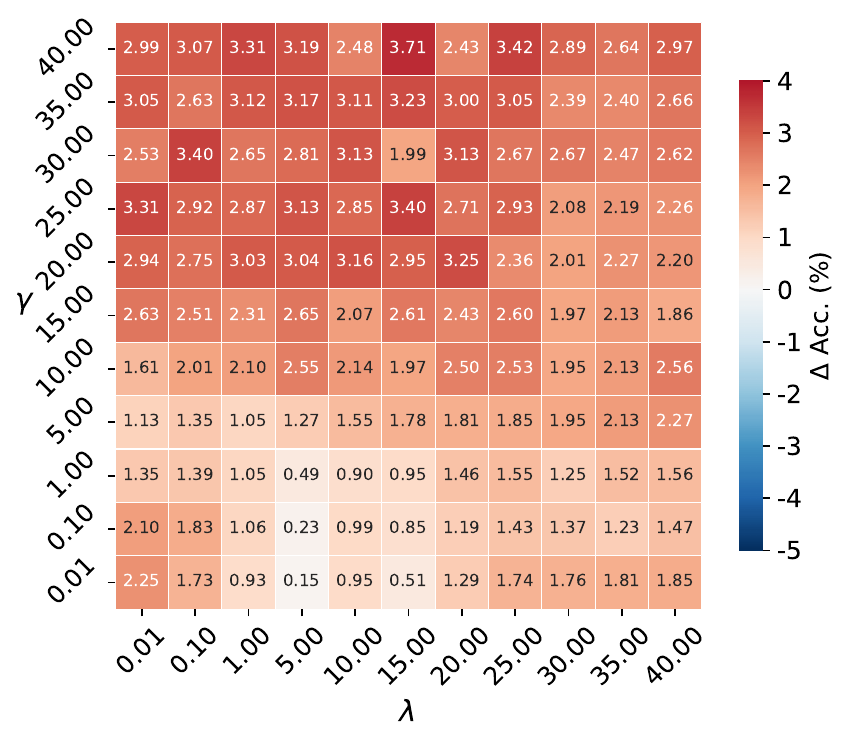}}
	\caption{Accuracy improvement (\%) over the no-alignment baseline ($\lambda=0$, $\gamma=0$) under different combinations of $\lambda$ and $\gamma$ for each alignment method.}
	\label{fig:lambda and gamma}
\end{figure}
\textbf{6) Sensitivity to Hyperparameters.}
We investigate the sensitivity of different alignment methods to hyperparameters $\lambda$ (prototype-level alignment weight) and $\gamma$ (instance-level alignment weight). Fig.~\ref{fig:lambda and gamma} shows the accuracy improvement over the baseline ($\lambda = 0$, $\gamma=0$) across a wide range of hyperparameter combinations on CIFAR-100 with HtFE$_9$ and Dir(0.1).

The heatmaps reveal a striking contrast between coordinate alignment and structural alignment. For MSE and Cosine, positive gains are observed only when both $\lambda$ and $\gamma$ are small. As the alignment strength increases, performance degrades substantially, with large hyperparameter values leading to significant negative transfer (up to -4.58\% for MSE). This phenomenon aligns with our theoretical analysis in Section~\ref{sec:coord-coupling}: when alignment strength is weak, the benefit from structural alignment may outweigh the harm from feature space enforcement. However, as alignment strength increases, the rigid coordinate matching increasingly suppresses model individuality, eventually dominating the overall effect and causing performance degradation. This observation provides empirical evidence that forcing clients into a unified feature space is indeed harmful in HtFL.

In contrast, GCSA and RCSA exhibit consistently positive improvements across nearly all hyperparameter combinations, with gains ranging from approximately 1\% to 4\%. This robustness stems from the fact that structural alignment decouples semantic structure transfer from feature space enforcement, allowing clients to benefit from global prototype knowledge without sacrificing their feature space individuality. Furthermore, the performance remains stable across a broad region around the optimal hyperparameters, indicating strong practical usability without the need for extensive hyperparameter tuning.

\begin{figure}
        \centering  
	\subfloat[CIFAR-10]{
		\includegraphics[width=\linewidth]{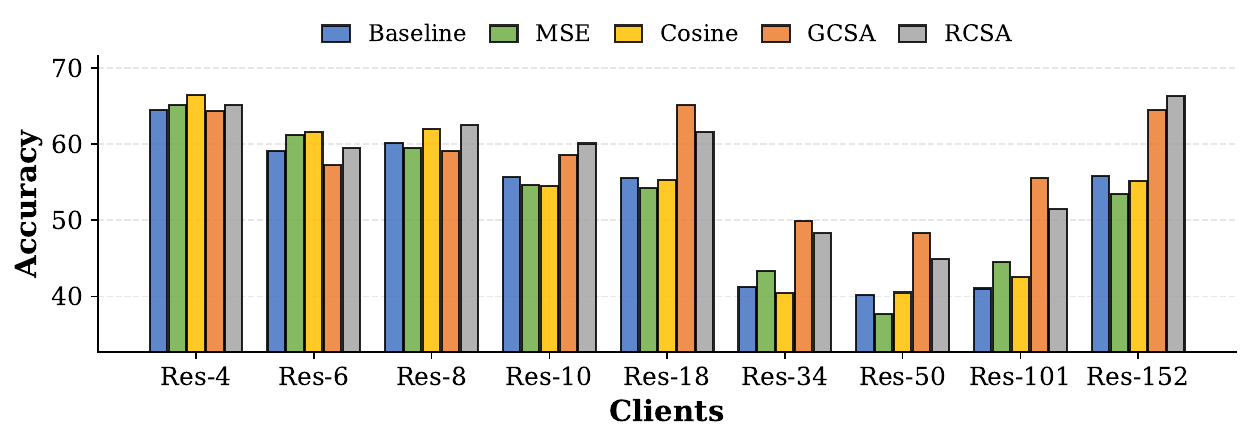}} \\
    \subfloat[CIFAR-100]{
		\includegraphics[width=\linewidth]{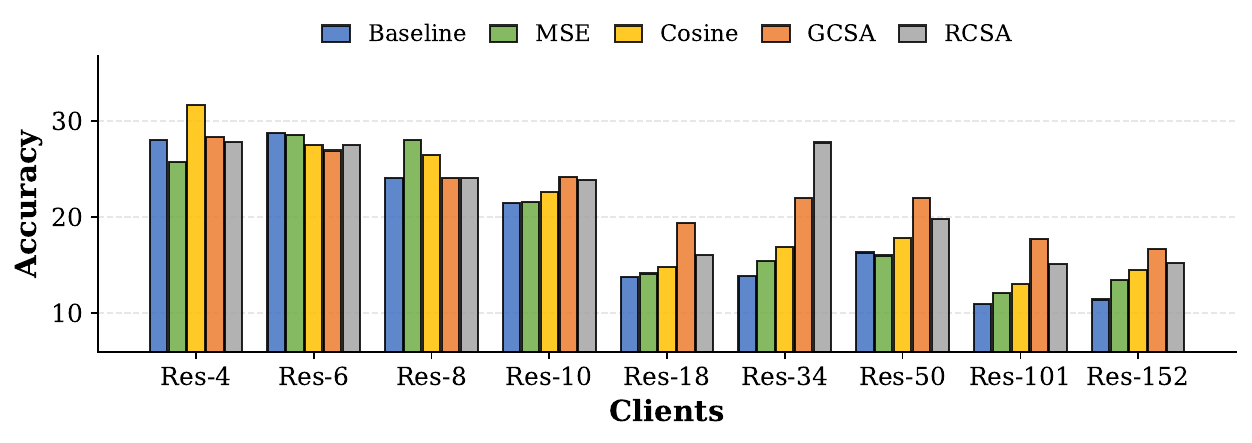}}
	\caption{Per-client test accuracy on CIFAR-10 and CIFAR-100 under Dir(0.1) with HtFE$_9$. Each client is assigned a ResNet variant of different depth.}
	\label{fig:acc per client}
\end{figure}
\textbf{7) Per-Client Performance Analysis.}
To validate our claim that coordinate alignment suppresses model learning capacity in heterogeneous settings, we analyze the performance of each individual client under HtFE$_9$, where nine clients are equipped with ResNet variants of increasing capacity (ResNet-4 to ResNet-152). Fig.~\ref{fig:acc per client} reports the test accuracy of each client on CIFAR-10 and CIFAR-100 under Dir(0.1).
The results reveal a clear pattern. For smaller models (ResNet-4 to ResNet-10), all alignment methods provide similar improvements over the baseline, as these models have limited representation capacity and benefit from any form of global guidance. However, for larger models (ResNet-18 and above), coordinate alignment (MSE and Cosine) yields only marginal gains over the baseline, and in some cases even underperforms it. In contrast, structural alignment (GCSA and RCSA) consistently delivers substantial improvements across all model capacities, with the advantage becoming more pronounced for deeper architectures.
This observation directly supports our motivation: larger models possess stronger representational power and naturally develop richer, more distinctive feature spaces. Forcing these models to match a shared global feature space constrains their expressiveness and prevents them from fully exploiting their capacity. Structural alignment, by targeting only on inter-class relational geometry, allows each client to leverage its architectural strength while still benefiting from global semantic knowledge. These results empirically confirm that preserving model individuality is crucial for effective knowledge transfer in HtFL.

\textbf{8) t-SNE Visualizations for Different Alignment Methods.}
To provide an intuitive understanding of the difference between coordinate alignment and structural alignment, we visualize the local prototypes from nine clients using t-SNE on CIFAR-10 with HtFE$_9$, as shown in Fig.~\ref{fig:prototype_visualization}. Different colors represent different clients, and different markers denote different classes.
The visualization reveals a clear distinction. Under coordinate alignment (MSE), prototypes are grouped by class, with the same class from different clients clustering together. This confirms that coordinate alignment enforces pointwise matching, pulling all clients toward a shared feature space regardless of their architectural differences. In contrast, under structural alignment (GCSA), prototypes are grouped by client, with each client's prototypes forming a distinct cluster. This demonstrates that structural alignment allows each client to optimize within its own feature subspace while maintaining consistent inter-class relations, corroborating our theoretical analysis and the quantitative results presented above.

\begin{figure}
        \centering  
	\subfloat[MSE]{
		\includegraphics[width=0.49\linewidth]{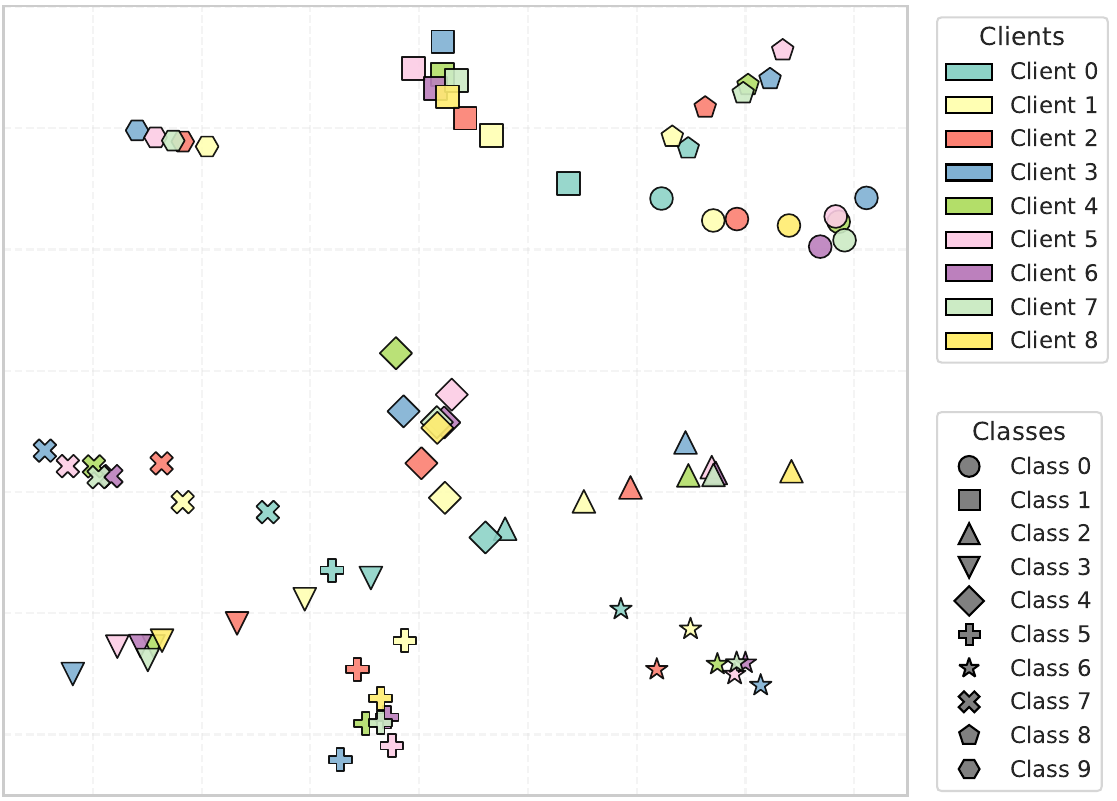}}
    \subfloat[GCSA]{
		\includegraphics[width=0.49\linewidth]{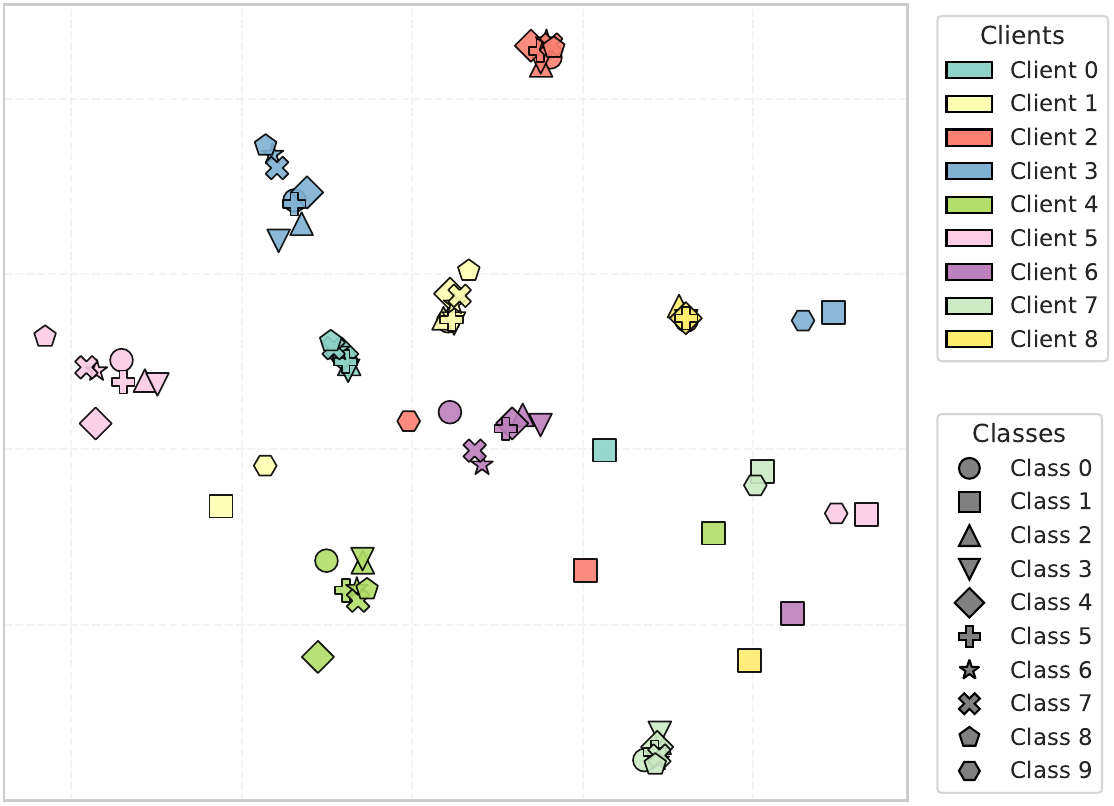}}
	\caption{t-SNE visualization of local prototypes under MSE (coordinate alignment) and GCSA (structural alignment) on CIFAR-10 with HtFE$_9$. Colors denote clients and markers denote classes.}
	\label{fig:prototype_visualization}
\end{figure}

\subsection{Combination with Advanced Global Prototype Construction Methods}


Our structure-level alignment framework is orthogonal to the choice of global prototype construction. In all previous experiments, we deliberately adopted the most basic aggregation-based strategy, where global prototypes are obtained by averaging local prototypes across clients. As discussed in Section~\ref{sec:preliminary}, most existing prototype-based HtFL methods focus on designing stronger global prototypes. In this subsection, we evaluate \methodname{} as a plug-in that can be combined with these advanced construction schemes.

\begin{figure}
        \centering  
	\subfloat[CIFAR-10]{
		\includegraphics[width=0.49\linewidth]{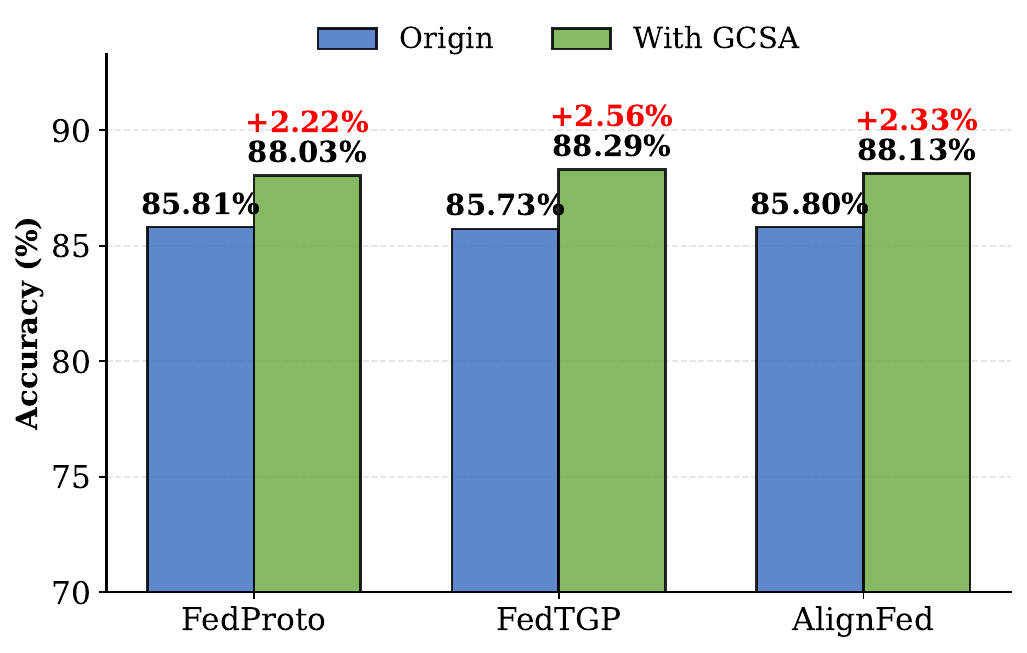}}
    \subfloat[CIFAR-100]{
		\includegraphics[width=0.49\linewidth]{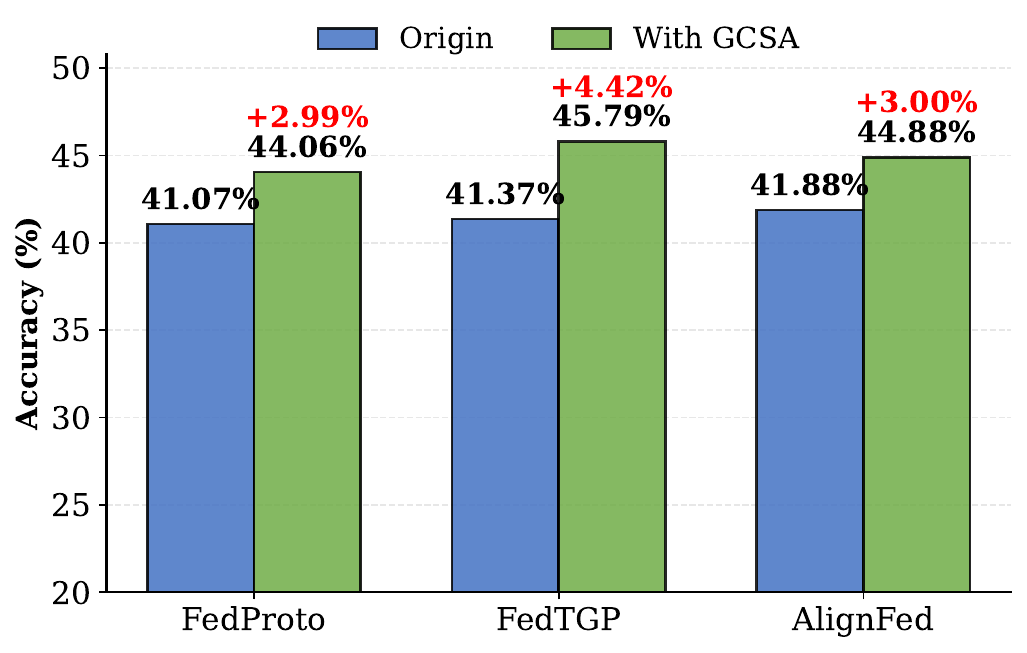}}
	\caption{Test accuracy (\%) of existing methods when combining with GCSA.}
	\label{fig:sota with gcsa}
\end{figure}

Specifically, we keep the global prototype pipelines of FedProto, FedTGP, and AlignFed unchanged, and only replace their original local alignment losses with GCSA. The results on CIFAR-10 and CIFAR-100 are reported in Fig.~\ref{fig:sota with gcsa}. Three key observations emerge from these experiments.

First, structural alignment consistently improves the accuracy of all three prototype-based baselines, confirming that it transfers prototype knowledge to clients more effectively than their original coordinate-level alignment.

Second, the performance gain from integrating GCSA reaches up to 4.42\%, which is substantially larger than the gain from improving prototype construction alone (0.81\%). This indicates that local alignment is a crucial yet largely overlooked component in prototype-based HtFL, and further highlights the practical value of the proposed structure-level alignment framework.

\subsection{Ablation Study}
\begin{table}[t] 
\caption{Experiments on the CIFAR-10 and CIFAR-100 to illustrate the effectiveness of each loss item.}
\label{expe:ablation}
	\begin{center}
				\begin{tabular}{lcccc}
\toprule
Settings & $\mathcal{L}_{\text{proto}}$ & $\mathcal{L}_{\text{inst}}$ & CIFAR-10 & CIFAR-100   \\ \midrule
\uppercase\expandafter{\romannumeral 1} & & & 84.87$\pm$\scriptsize0.15 & 40.55$\pm$\scriptsize0.08  \\
\uppercase\expandafter{\romannumeral 2}& \checkmark & & $86.45\pm$\scriptsize0.11 & 42.47$\pm$\scriptsize0.28 \\
\uppercase\expandafter{\romannumeral 3} &  & \checkmark &  87.02$\pm$\scriptsize0.13 & 43.27$\pm$\scriptsize 0.19 \\
\uppercase\expandafter{\romannumeral 4} & \checkmark & \checkmark  & 88.03$\pm$\scriptsize0.13 & 44.06$\pm$\scriptsize 0.16 \\

\bottomrule
\end{tabular}
	\end{center}
\end{table}
We conduct ablation experiments to evaluate the contribution of each component in our structural alignment framework. We consider four configurations on CIFAR-10 and CIFAR-100 under Dir(0.1) with HtFE$_9$ using GCSA as the structural alignment instantiation. The results are reported in Table~\ref{expe:ablation}.

Comparing Setting I (baseline without any alignment) to Settings II and III, we observe that both $\mathcal{L}_{\text{proto}}$ and $\mathcal{L}_{\text{inst}}$ individually improve performance. On CIFAR-10, $\mathcal{L}_{\text{proto}}$ alone yields a 1.58\% improvement, while $\mathcal{L}_{\text{inst}}$ alone provides a 2.15\% gain. On CIFAR-100, the improvements are 1.92\% and 2.72\%, respectively. This confirms that both prototype-level and instance-level structural alignment contribute positively to knowledge transfer.

Comparing Setting IV to Settings II and III, combining both losses achieves the best performance, with 88.03\% on CIFAR-10 and 44.06\% on CIFAR-100. The two losses provide complementary guidance: $\mathcal{L}_{\text{proto}}$ enforces consistency in class-level semantic geometry, while $\mathcal{L}_{\text{inst}}$ further refines individual sample representations to conform to the global prototype structure. Together, they deliver comprehensive structural guidance at both the class and instance levels.


\subsection{Convergence Speed}
\begin{figure}
        \centering  
	\subfloat[CIFAR-10]{
		\includegraphics[width=0.49\linewidth]{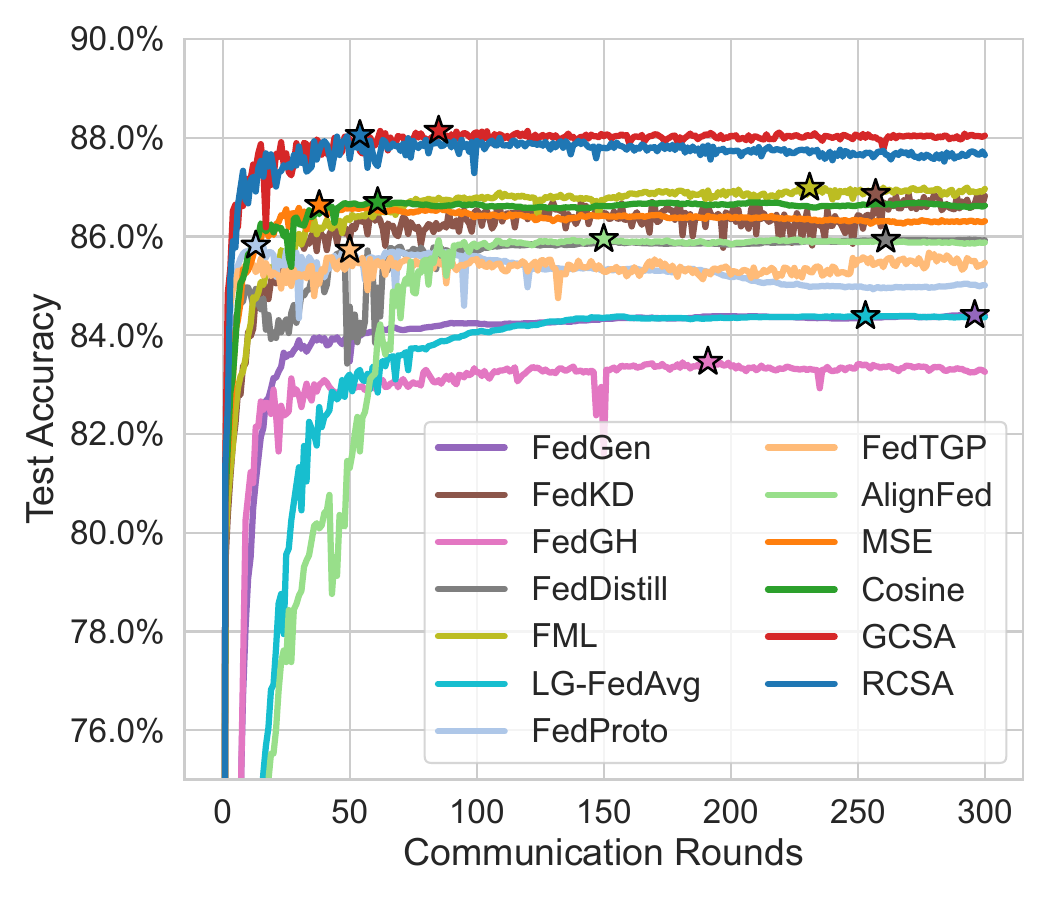}}
    \subfloat[CIFAR-100]{
		\includegraphics[width=0.49\linewidth]{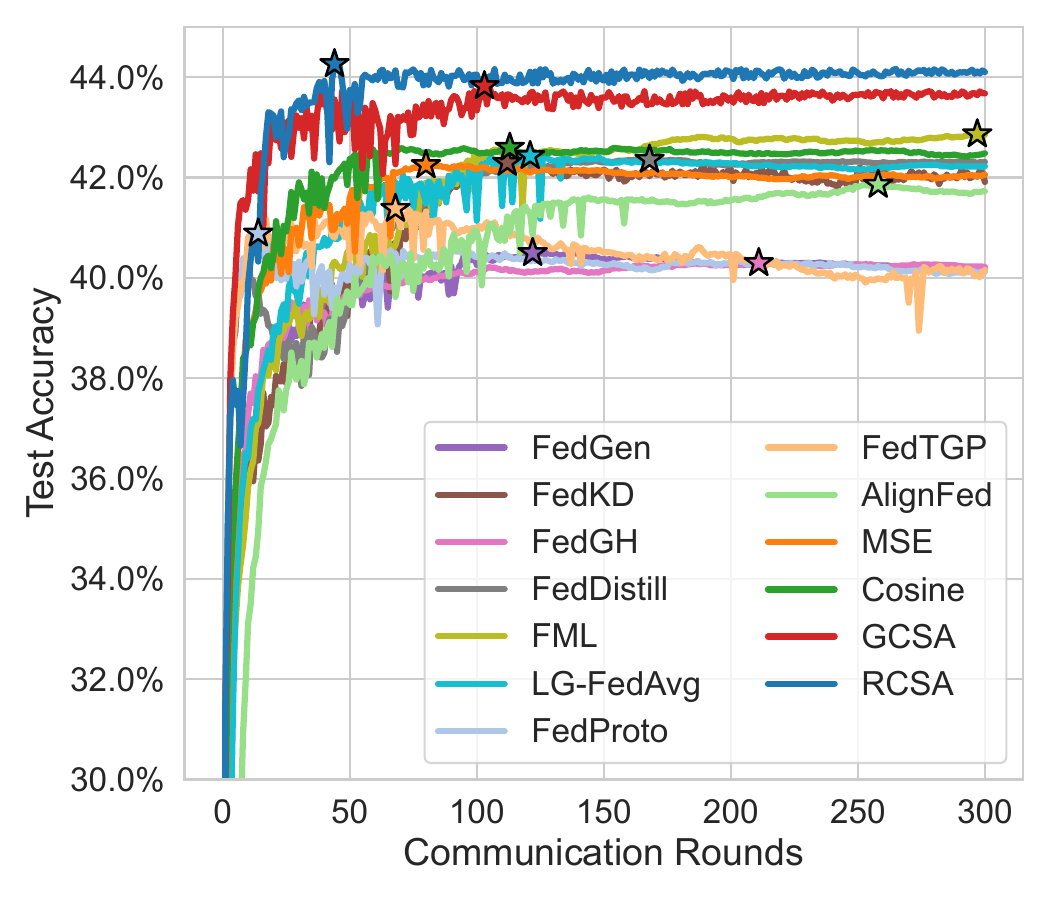}}
	\caption{The test accuracy curve of different methods on CIFAR-10 and CIFAR-100. Best accuracy is marked as star.}
	\label{fig:convergence}
\end{figure}

In this section, we examine the convergence behavior of different methods on CIFAR-10 and CIFAR-100 under the Dir(0.1) label shift setting with HtFE9, as shown in Fig.~\ref{fig:convergence}.

Our structural alignment methods (GCSA and RCSA) achieve both fast convergence and high final accuracy. Among prototype-based methods, FedProto, FedTGP, AlignFed, and coordinate alignment baselines (MSE, Cosine) converge at similar rates but plateau at noticeably lower accuracy levels. In contrast, some non-prototype HtFL methods such as FedKD and FML eventually reach competitive accuracy, yet their convergence is considerably slower, reflecting the higher computational overhead of auxiliary models or distillation modules. Our structural alignment combines the communication efficiency of prototype-based methods with improved knowledge transfer, achieving strong final performance without sacrificing convergence speed.

\subsection{Training Efficiency}
We compare the communication and computation costs of different methods on CIFAR-100 with HtFE$_4$, as shown in Table~\ref{tab:efficiency}.

Our method (Ours-GCSA) inherits the communication efficiency of prototype-based approaches. Compared to other HtFL methods that rely on auxiliary models or full-model distillation (e.g., FML, FedKD, FedMRL), our method requires significantly less communication overhead.
Regarding computation, our method introduces a modest increase in client-side training time compared to existing prototype-based methods (e.g., 36.4s vs. 32.8s for FedProto) due to the two-level alignment design. However, this overhead is marginal and does not become a system bottleneck, especially considering the substantial accuracy gains demonstrated in previous experiments.
\begin{table}[tbp]
\centering
\caption{Communication and Computation Costs on CIFAR100 with HtFE$_4$.}
\label{tab:efficiency}
\begin{tabular}{lccccc}
\toprule
\multirow{2.5}{*}{\textbf{Method}} & \multicolumn{2}{c}{\textbf{Comm. (MB)}} & \multicolumn{2}{c}{\textbf{Computation (s)}} \\
\cmidrule(lr){2-3} \cmidrule(lr){4-5}
& \textbf{Up.} & \textbf{Down.} & \textbf{Server} & \textbf{Client} \\
\midrule
\multicolumn{6}{c}{Prototype-based HtFL Methods} \\
\midrule
FedProto   & 1.69  & 3.90  & 0.118 & 32.800 \\
FedTGP     & 1.69  & 3.90  & 1.235 & 29.640 \\
AlignFed   & 3.91  & 3.91  & 0.006 & 28.768 \\
\midrule
Ours-MSE   & 1.69  & 3.90  & 0.060 & 35.698 \\
\textbf{Ours-GCSA}  & 1.69  & 3.90  & 0.060 & 36.400 \\
\midrule
\multicolumn{6}{c}{Other HtFL Methods} \\
\midrule
FedGen     & 3.91  & 25.20 & 4.036 & 29.600 \\
FedGH      & 1.69  & 3.91  & 0.955 & 25.924 \\
LG-FedAvg  & 3.91  & 3.91 & 1.936 & 26.496 \\
FML        & 70.54 & 70.54 & 0.364 & 37.428 \\
FedDistill & 0.33  & 0.76  & 0.089 & 35.815 \\
FedKD      & 62.99 & 62.99 & 2.180 & 37.184 \\
FedMRL     & 70.54 & 70.54 & 0.343 & 26.780 \\
\bottomrule
\end{tabular}
\end{table}

\subsection{Robustness to Local Epoch $E$}
The number of local training epochs $E$ is a key hyperparameter in FL. To evaluate its impact on different methods, we vary $E \in \{1, 5, 10, 20\}$ on CIFAR-100 under Dir(0.5) with HtFE$_9$. As shown in Fig.~\ref{fig:local_epoch}, we make two main observations: (1) Most methods exhibit performance degradation as $E$ increases, since less frequent communication reduces the effectiveness of global knowledge exchange. (2) Our structural alignment methods (GCSA and RCSA) remain robust across all settings and consistently outperform other HtFL methods under all local epoch configurations.
\begin{figure}[tb]
		\centerline{\includegraphics[width=\linewidth]{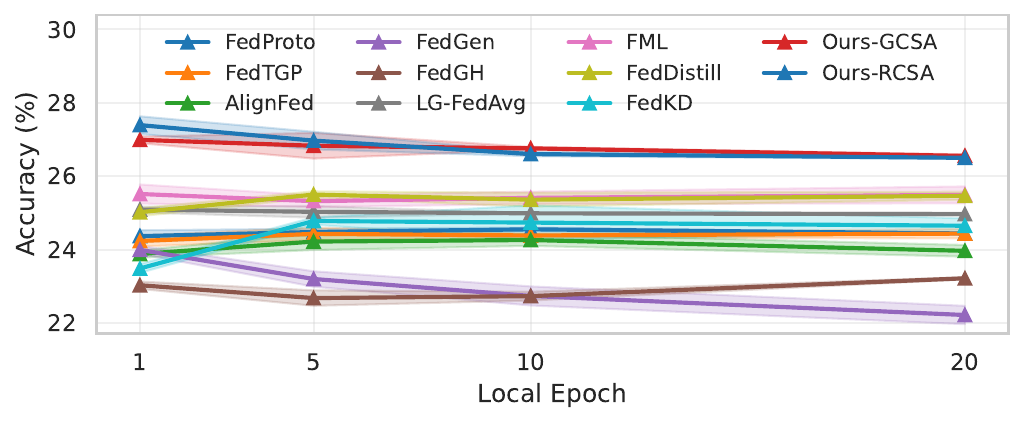}}
        \caption{Effect of local epoch $E$ on different methods.}
		\label{fig:local_epoch}
\end{figure}

\section{Discussion}
\subsection{Relationship with Contrastive Alignment}
\label{sec:contrastive_discussion}

Many prototype-based methods adopt contrastive losses (\eg, InfoNCE) rather than direct cosine matching for prototype alignment. To clarify why we use cosine-based alignment as the coordinate baseline, we analyze the contrastive loss and show that it conflates two orthogonal mechanisms.

The standard contrastive loss for aligning a representation $z_i$ with its corresponding prototype $P_{y_i}$ is formulated as:
\begin{equation}
    \mathcal{L}_{con} = -\log \frac{\exp(\text{sim}(z_i, P_{y_i}) / \tau)}{\sum_{j=1}^{C} \exp(\text{sim}(z_i, P_j) / \tau)}
\end{equation}
where $\text{sim}(\cdot, \cdot)$ denotes cosine similarity and $\tau$ is a temperature parameter. This loss can be decomposed into two distinct components:
\begin{equation}
    \mathcal{L}_{con} = \underbrace{-\frac{1}{\tau}\text{sim}(z_i, P_{y_i})}_{\text{Alignment Term}} + \underbrace{\log \left( \sum_{j=1}^{C} \exp(\text{sim}(z_i, P_j) / \tau) \right)}_{\text{Uniformity Term}}
\end{equation}

\textit{(1) Coordinate Alignment Term:} The first term is strictly equivalent to cosine-based coordinate alignment scaled by $1/\tau$.
    
\textit{(2) Uniformity Term:} The second term acts as a regularizer that encourages separation among inter-class prototypes, thereby improving the uniformity of the embedding space. This term governs the \emph{global geometry of prototypes} and is \textit{orthogonal} to the \textit{alignment mechanism} (\ie, how $z_i$ matches $P_{y_i}$), which is the primary focus of this work.

By substituting contrastive objectives with direct cosine alignment, we isolate the alignment mechanism from uniformity regularization, ensuring that our experimental comparisons strictly evaluate the efficacy of structural versus coordinate alignment without the uniformity effect acting as a confounding variable. We note that our proposed structural alignment (GCSA/RCSA) is fully compatible with uniformity constraints. One can incorporate the uniformity term alongside our structural loss to further refine the global prototype geometry, though this extension is outside the primary scope of our analysis.
\section{Conclusion}
In this paper, we identify a critical yet overlooked limitation in existing prototype-based HtFL methods: the reliance on coordinate alignment, which forces all clients to map their representations into a shared global feature subspace and suppresses model learning capacity. We reveal that coordinate alignment implicitly couples two distinct objectives: aligning inter-class semantic structure, which is beneficial for knowledge transfer, and enforcing a shared feature basis, which is unnecessary and harmful under model heterogeneity. To decouple these effects, we propose \methodname{}, which shifts the alignment objective from absolute coordinates to relational geometry. We provide two concrete instantiations, GCSA and RCSA, that effectively transfer global semantic knowledge while preserving the unique feature spaces of heterogeneous client models. Extensive experiments across diverse datasets, non-IID scenarios, modalities, and model heterogeneity levels demonstrate that \methodname{} consistently outperforms state-of-the-art methods. Furthermore, our framework is orthogonal to prototype construction strategies and can serve as a plug-and-play module to enhance existing prototype-based HtFL methods.

\bibliographystyle{IEEEtran}
\bibliography{main}

\appendix
\subsection{Proof of Proposition 1}
\label{app:proof_prop1}

\begin{proof}
Recall the coordinate alignment loss
\begin{equation}
\mathcal{L}_{coord}(Z,P) := \|\hat Z - \hat P\|_F^2,
\end{equation}
and define
\begin{equation}
R^* \in \arg\min_{R\in\mathcal{O}(d)} \|\hat Z - \hat P R\|_F^2
 \Longleftrightarrow 
R^* \in \arg\max_{R\in\mathcal{O}(d)} \langle \hat Z, \hat P R\rangle,
\end{equation}
where $\langle A,B\rangle := \mathrm{tr}(A^\top B)$.

Using the identity $\|A-B\|_F^2 = \|A\|_F^2 + \|B\|_F^2 - 2\langle A,B\rangle$, we have
\begin{equation}
\mathcal{L}_{coord}(Z,P)
= \|\hat Z\|_F^2 + \|\hat P\|_F^2 - 2\langle \hat Z, \hat P\rangle.
\end{equation}
Similarly, for any orthogonal $R\in\mathcal{O}(d)$,
\begin{align}
&\|\hat Z - \hat P R\|_F^2 \nonumber \\
&= \|\hat Z\|_F^2 + \|\hat P R\|_F^2 - 2\langle \hat Z, \hat P R\rangle \nonumber \\
&= \|\hat Z\|_F^2 + \|\hat P\|_F^2 - 2\langle \hat Z, \hat P R\rangle,
\end{align}
where we used $\|\hat P R\|_F = \|\hat P\|_F$ since $R$ is orthogonal.

Therefore,
\begin{align}
\mathcal{L}_{shape}(Z,P)
&:= \min_{R\in\mathcal{O}(d)} \|\hat Z - \hat P R\|_F^2 \\
&= \|\hat Z\|_F^2 + \|\hat P\|_F^2 - 2\max_{R\in\mathcal{O}(d)} \langle \hat Z, \hat P R\rangle \\
&= \|\hat Z\|_F^2 + \|\hat P\|_F^2 - 2\langle \hat Z, \hat P R^*\rangle .
\end{align}
Combining the above expressions yields the exact decomposition
\begin{align}
\mathcal{L}_{coord}(Z,P)
&= \mathcal{L}_{shape}(Z,P) + 2\big(\langle \hat Z, \hat P R^*\rangle - \langle \hat Z, \hat P\rangle \big) \nonumber \\
&=: \mathcal{L}_{shape}(Z,P) + \mathcal{L}_{rigid}(Z,P),
\end{align}
which proves Eq.~(9).

Finally, by the optimality of $R^*$ for the Procrustes maximization,
\begin{equation}
\langle \hat Z, \hat P R^*\rangle \ge \langle \hat Z, \hat P I\rangle = \langle \hat Z, \hat P\rangle,
\end{equation}
hence $\mathcal{L}_{rigid}(Z,P)\ge 0$, and $\mathcal{L}_{rigid}(Z,P)=0$ if and only if the identity rotation
achieves the maximum, i.e., $\langle \hat Z, \hat P\rangle = \max_{R\in\mathcal{O}(d)}\langle \hat Z, \hat P R\rangle$.
\end{proof}

\subsection{Proof of Proposition~2}
\begin{proof}
Let $\bar{P} \in \mathbb{R}^{1 \times d}$ denote the row-wise mean of $P$, and define the centered matrix
$
P_c := P - \mathbf{1}\bar{P}.
$
Similarly, let $\bar{Q}$ be the row-wise mean of $Q$ and $Q_c := Q - \mathbf{1}\bar{Q}$.
By the assumed transformation $Q = \alpha P R + \mathbf{1} b^\top$ with $\alpha>0$ and $R^\top R = I$, we have
\begin{align}
\bar{Q}
&= \frac{1}{n}\mathbf{1}^\top Q
= \frac{1}{n}\mathbf{1}^\top(\alpha P R + \mathbf{1} b^\top)
= \alpha \bar{P} R + b^\top.
\end{align}
Therefore,
\begin{align}
Q_c
&= Q - \mathbf{1}\bar{Q} \nonumber \\
&= (\alpha P R + \mathbf{1} b^\top) - \mathbf{1}(\alpha \bar{P} R + b^\top) \nonumber \\
&= \alpha (P - \mathbf{1}\bar{P}) R \nonumber \\
&= \alpha P_c R.
\end{align}
Now consider the centered Gram matrices:
\begin{align}
K_P &= P_c P_c^\top, \\
K_Q &= Q_c Q_c^\top \nonumber \\
&= (\alpha P_c R)(\alpha P_c R)^\top \nonumber \\
&= \alpha^2 P_c R R^\top P_c^\top \nonumber \\
&= \alpha^2 P_c P_c^\top \nonumber \\
&= \alpha^2 K_P.
\end{align}
Hence $K_Q$ is a positive scalar multiple of $K_P$. The cosine similarity between $K_P$ and $K_Q$ equals
\begin{align}
\frac{\langle K_P, K_Q \rangle}{\lVert K_P \rVert_F \lVert K_Q \rVert_F} 
&= \frac{\langle K_P, \alpha^2 K_P \rangle}{\lVert K_P \rVert_F \cdot \lVert \alpha^2 K_P \rVert_F} \\
&= \frac{\alpha^2 \lVert K_P \rVert_F^2}{\lVert K_P \rVert_F \cdot \alpha^2 \lVert K_P \rVert_F} \nonumber \\
&= 1, \nonumber
\end{align}
whenever $\lVert K_P\rVert_F \neq 0$ (the degenerate case $\lVert K_P\rVert_F=0$ implies all rows of $P$ are identical after centering and the alignment is trivially satisfied).
Thus,
$
\mathcal{L}_{\mathrm{GCSA}}(P,Q) = 1 - 1 = 0.
$
\end{proof}

\subsection{Proof of Proposition~3}
\begin{proof}
Assume $Q = P R$ with $R^\top R = I$.
Let $\tilde{P}$ and $\tilde{Q}$ denote row-wise $\ell_2$ normalized versions of $P$ and $Q$:
\[
\tilde{P}[i,:] = \frac{P[i,:]}{\lVert P[i,:]\rVert_2}, 
\qquad
\tilde{Q}[i,:] = \frac{Q[i,:]}{\lVert Q[i,:]\rVert_2}.
\]
Because $R$ is orthogonal, it preserves Euclidean norms, so for every row $i$,
\[
\lVert Q[i,:]\rVert_2 = \lVert P[i,:]R\rVert_2 = \lVert P[i,:]\rVert_2,
\]
and therefore
\[
\tilde{Q}[i,:] = \frac{P[i,:]R}{\lVert P[i,:]\rVert_2} = \tilde{P}[i,:]R,
\quad\text{or equivalently}\quad
\tilde{Q} = \tilde{P}R.
\]
Now consider the squared-distance RDM entries used in RCSA:
\begin{align}
\mathrm{RDM}_Q[i,j]
&= \big\lVert \tilde{Q}[i,:] - \tilde{Q}[j,:] \big\rVert_2^2
= \big\lVert \tilde{P}[i,:]R - \tilde{P}[j,:]R \big\rVert_2^2 \nonumber \\
&= \big\lVert (\tilde{P}[i,:]-\tilde{P}[j,:])R \big\rVert_2^2
= \big\lVert \tilde{P}[i,:]-\tilde{P}[j,:] \big\rVert_2^2  \nonumber \\
&= \mathrm{RDM}_P[i,j],
\end{align}
where we used norm preservation under orthogonal transforms in the third equality.
Thus $\mathrm{RDM}_Q = \mathrm{RDM}_P$, which implies their upper-triangular vectorizations are identical:
$
\mathrm{vec}(\mathrm{RDM}_Q) = \mathrm{vec}(\mathrm{RDM}_P).
$
Hence the cosine similarity between these two vectors equals $1$ (assuming the vector is nonzero; otherwise the loss is trivially $0$), and therefore
\[
\mathcal{L}_{\mathrm{RCSA}}(P,Q) = 1 - 1 = 0.
\]
\end{proof}

\vfill

\end{document}